\documentclass{article}

\usepackage{microtype}
\usepackage{graphicx}
\usepackage{booktabs} 

\usepackage{cite}
\usepackage{amsmath,amssymb,amsfonts}
\usepackage[algo2e,ruled,vlined,linesnumbered]{algorithm2e}
\usepackage{graphicx}
\usepackage{textcomp}
\usepackage{booktabs}
\usepackage{xcolor}
\usepackage{hyperref}
\usepackage{booktabs}
\usepackage{easyReview}
\usepackage{enumitem}
\usepackage{floatrow}
\usepackage{subcaption}
\usepackage{ulem}
\usepackage[hang,flushmargin]{footmisc} 

\newtheorem{theorem}{Theorem}[section]

\newtheorem{claim}[theorem]{Claim}

\newenvironment{proof}{\noindent{\bf Proof.}}{\hfill$\square$\medskip}

\newenvironment{proof*}[1]{\noindent{\bf Proof of #1.}}{\hfill$\square$\medskip}

\usepackage{hyperref}

\usepackage[accepted]{mlsys2020}

\mlsystitlerunning{TT-Rec: Tensor Train Compression for Deep Learning Recommendation Models}

\begin{document}

\twocolumn[
\mlsystitle{TT-Rec: Tensor Train Compression \\ for Deep Learning Recommendation Model Embeddings}


\mlsyssetsymbol{equal}{*}

\begin{mlsysauthorlist}
\mlsysauthor{Chunxing Yin}{goo,to}
\mlsysauthor{Bilge Acun}{to}
\mlsysauthor{Xing Liu}{to}
\mlsysauthor{Carole-Jean Wu}{to}

\mlsysauthor{$^{1}$Georgia Institute of Technology, $^{2}$Facebook AI}{}
\end{mlsysauthorlist}

\mlsysaffiliation{to}{Georgia Institute of Technology, USA}
\mlsysaffiliation{goo}{Facebook AI, USA}
\mlsyscorrespondingauthor{Bilge Acun}{acun@fb.com}

\mlsyskeywords{Machine Learning, MLSys, Tensor Train Compression, Embedding}

\vskip 0.3in


\begin{abstract}
The memory capacity of embedding tables in deep learning recommendation models (DLRMs) is increasing dramatically from tens of GBs to TBs across the industry.
Given the fast growth in DLRMs, novel solutions are urgently needed, in order to enable fast and efficient DLRM innovations. At the same time, this must be done without having to exponentially increase infrastructure capacity demands.
In this paper, we demonstrate the promising potential of Tensor Train decomposition for DLRMs (TT-Rec), an important yet under-investigated context.
We design and implement optimized kernels (TT-EmbeddingBag) to evaluate the proposed TT-Rec design. 
TT-EmbeddingBag is 3$\times$ faster than the SOTA TT implementation.
The performance of TT-Rec is further optimized with the batched matrix multiplication and caching strategies for embedding vector lookup operations. In addition, we present mathematically and empirically the effect of weight initialization distribution on DLRM accuracy and propose to initialize the tensor cores of TT-Rec following the sampled Gaussian distribution.
We evaluate TT-Rec across three important design space dimensions---memory capacity, accuracy, and timing performance---by training MLPerf-DLRM with Criteo's Kaggle and Terabyte data sets.
TT-Rec achieves 117$\times$ and 112$\times$ model size compression, for Kaggle and Terabyte, respectively. This impressive model size reduction can come with no accuracy nor training time overhead as compared to the uncompressed baseline.

Our code is available on Github at \href{https://github.com/facebookresearch/FBTT-Embedding}{facebookresearch/FBTT-Embedding}.
\end{abstract}
]


\section{Introduction}
Deep neural networks (DNNs) are witnessing an unprecedented growth in all dimensions: data, model complexity, and the cost of infrastructure required for model training and deployment.
For instance, at Facebook, the amount of data used in machine learning
tripled in one year (2019--20), which led to an eight-fold increase in the amount of computation required for training~\cite{hazelwood:sparkAI2020}.
Similarly, the number of parameters in state-of-the-art language models have increased exponentially, currently at over 175 billion parameters in OpenAI’s GPT-3~\cite{brown2020language}.
In response, there is considerable interest to design domain-specific accelerators and at-scale infrastructures~\cite{jouppi:ISCA2017,ovtcharov2015toward,chung2018serving,fowers2018a,nvidia:DGXa100,mlperf-inference,mlperf-training,zion,hazelwood:hpca2018,Alibaba_Hanguang,Amazon_inferentia}.
But to achieve significant reductions in the cost and capacity, we still need to discover orders-of-magnitude reductions in the infrastructure demand while maintaining or even outperforming SOTA model accuracy.

In this work, we consider a new algorithmic approach to cope with the large memory requirement of DNNs, focusing on the critical use-case of embedding tables in deep learning-based recommendation models (DLRMs).
These models represent one of the most resource-demanding deep learning workloads, consuming more than 50\% of training and 80\% of the total AI inference cycles at Facebook’s data centers~\cite{Gupta2019,naumov:arxiv2020}. From systems perspective, the large embedding tables that contribute to more than 99\% of the total recommendation model capacity are the Amdahl’s bottleneck for optimization. 
Our approach uses \textit{tensorization} to address the large memory capacity demand of embedding tables in a DLRM. 
At a high level, tensorization replaces layers of a neural network with an approximate and structured low-rank form~\cite{Novikov2015}. This form is parametric -- its ``shape'' determines the design trade-off between storage capacity, execution time, and model accuracy. Furthermore, tensorized representation can be fine-tuned with respect to the architecture of a given hardware platform. 
Figure~\ref{fig:motivation} illustrates the design space with respective to the tunable parameters, such as rank of the tensorization method, dimensions of the embedding and number of tables to compress. Data points on the Pareto frontier (black curve) represent the optimal settings that maximize the recommendation model accuracy (y-axis) given a corresponding memory size (x-axis). The parameters of these optimal data points vary, depending on the model's characteristics (embedding dimensions), the tensorization setting (i.e., tensor ranks), and the memory capacity of the underlying training system.
Given the large configuration space, parameters need to be
carefully studied in order to achieve highest possible model
quality given a target model size.


We design a \textbf{T}ensor-\textbf{T}rain compression technique for deep learning \textbf{Rec}ommendation models, called \textit{TT-Rec}.
The core idea is to replace large embedding tables in a DLRM with a sequence of matrix products. This method is analogous to techniques that use lookup tables to trade-off memory storage and bandwidth with computation.
TT-Rec suits well for accelerators like GPUs, which have a relatively higher compute-to-memory (FLOPs-per-Byte) ratio and limited memory capacity. 
Since the tensor representation is a ``lossy'' compression scheme, to compensate for accuracy loss
we propose a new way to initialize the element distribution of the tensor form.
Furthermore, to mitigate increases in training time when the tensor form must be decompressed, we introduce a cache structure that exploits the unique sparse feature distribution in DLRMs, in which we store the most accessed embedding vectors in the uncompressed format.
Since these cached embedding vectors are learned without compression, using this cache design can also help recovering model accuracy.
Thus, TT-Rec uses a hybrid approach to learn features
and deliver on-par model accuracy while requiring orders-of-magnitude less memory capacity.

\begin{figure}[t]
    \centering
    \includegraphics[width=.8\columnwidth]{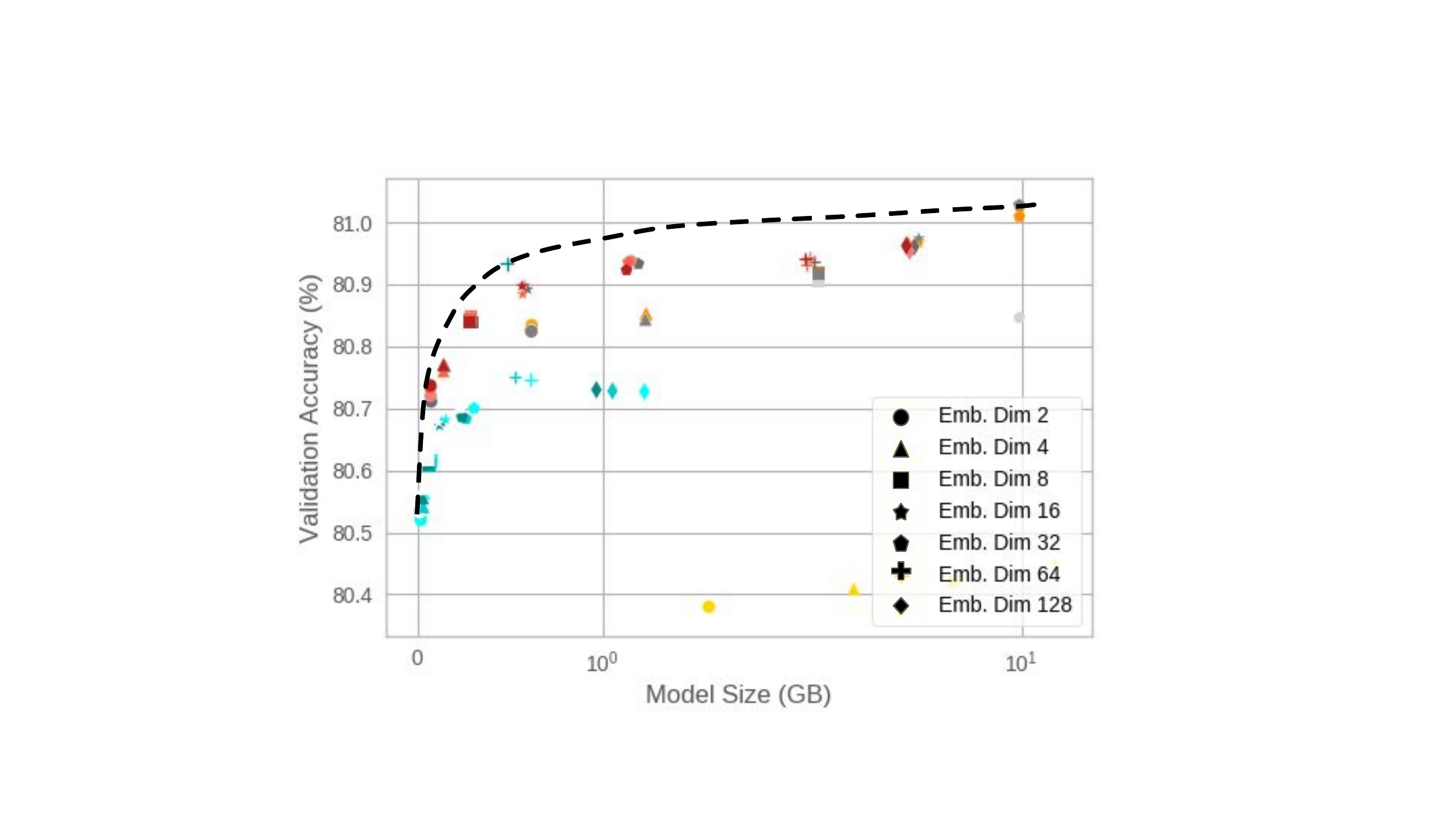}
    \vspace{-1mm}
    \caption{
    The design space demonstrates the potential of DLRM accuracy and model size tradeoff with respect to the tunable parameters, e.g., ranks of the tensorization method (colors), dimensions of embedding (shapes), and number of compressed embedding tables  (brightness). The data points that fall onto the Pareto Frontier (the black curve) represent optimal settings that maximize DLRM accuracy (the y-axis) given a memory size (the x-axis).
    }
    \label{fig:motivation}
    \vspace{-5mm}
\end{figure}

We show significant compression ratios and improved training time performance at the same time with a judicious design and parameterization of the tensor-train compression technique. The compute-to-memory ratio of the underlying hardware can also be taken into account when parameterizing the proposed technique. 
While prior works have demonstrated tensor-train compression techniques for embedding layers in language models~\cite{hrinchuk2020tensorized}, this paper is the first to explore and customize tensor-train compression techniques for DLRMs, with a particular focus on minimizing the significant memory capacity requirement of the embedding layers (over tens to hundreds of GBs, or over 99\% of the total model size).

The main contributions of this paper are as follows:

\setlist{nolistsep}
\begin{itemize}[leftmargin=*,noitemsep]
    \item This work applies tensor-train compression in a new application context, compressing the embedding layers of DLRMs.
    \item Our in-depth design space characterization shows the importance of choosing the right number of embedding tables to compress and the dimension of the compressed tensors. In particular, we quantify the potential trade-off between memory requirements and accuracy. 
    \item To recover accuracy loss, we propose to use a sampled Gaussian distribution for the weight initialization of the tensor cores. Furthermore, to accelerate the training performance of TT-Rec, we introduce a separate cache structure to store frequently-accessed embedding vectors in the uncompressed format, which we show empirically helps accuracy improvement.
    \item We demonstrate the promise of TT-Rec: on Terabyte, TT-Rec achieves higher model accuracy (0.19\% to 0.42\% over the baseline) while reducing the total memory requirement of the embedding tables by 22$\times$ to $112\times$ 
    with a small amount of 10\% training time increase on average, and similarly for Kaggle.
    \item TT-Rec offers a flexible design space between memory capacity,  training time and model accuracy. It is an effective approach especially for online recommendation training. The orders-of-magnitude lower memory requirement with TT-Rec also unlocks a range of modern AI training accelerators for DLRM training.   
\end{itemize}

\section{Background}
\paragraph{Deep learning recommendation models.}




Figure~\ref{fig:tt-dlrm-arch} depicts the generalized model architecture for DLRMs. The shaded structure (with the orthogonal green stripes) represents the baseline while the other shaded structure (with the solid yellow box) depicts the proposed TT-Rec design. There are  two  primary  components: the  Multi  Layer  Perceptron  (MLP)  layer  modules  and  the Embedding Tables (EMBs). The MLP layers are used to process continuous features, such as user age, while the EMBs are used to process categorical features by encoding sparse, high-dimensional inputs into dense, vector representation. 
The encoded vectors, usually of length 16 to 128 in industry-scale recommender systems, are processed by an interaction operation followed by a top MLP layer.
\begin{figure}[t]
    \centering
    \includegraphics[width=\columnwidth]{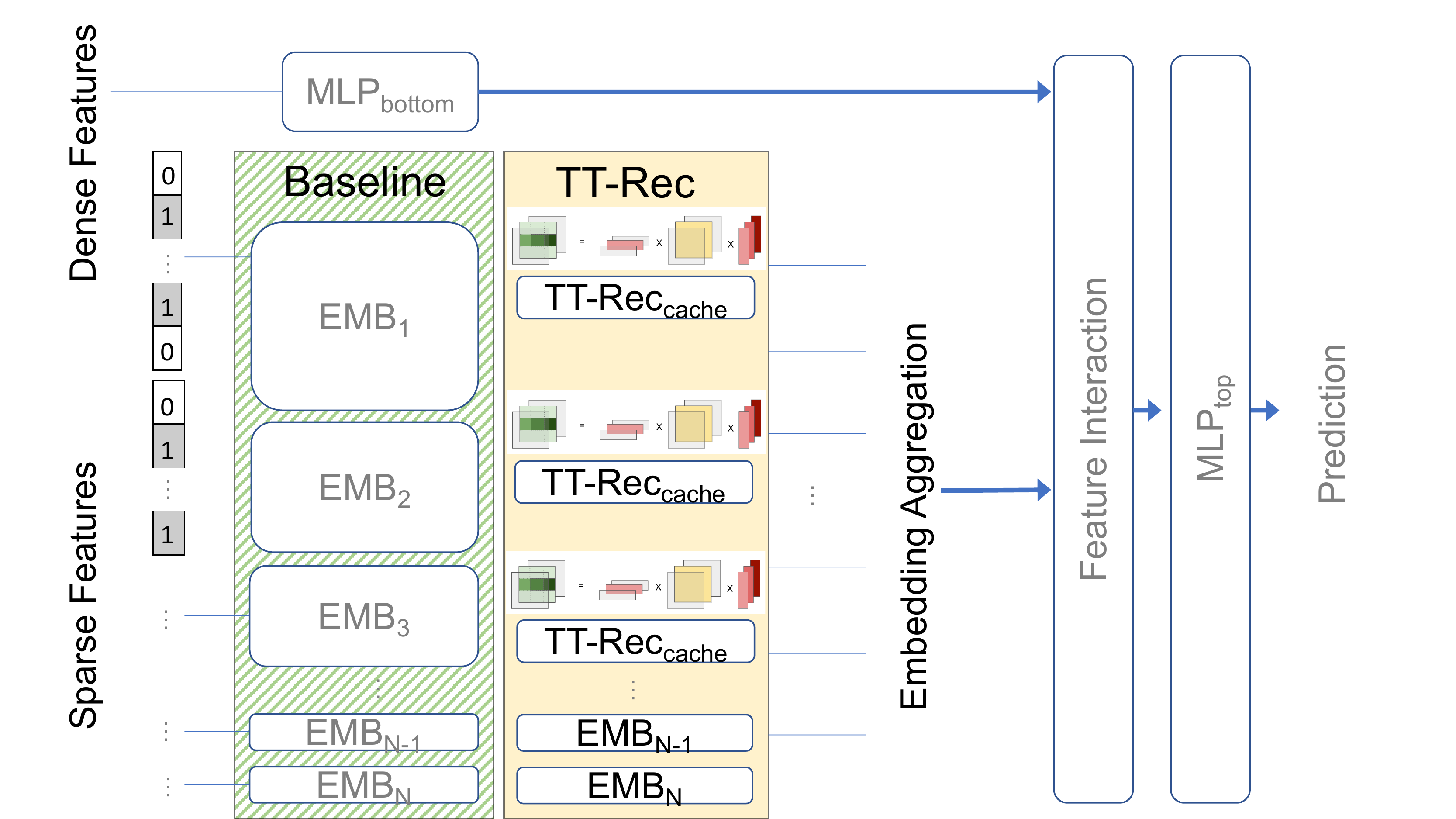}
    \vspace{-5mm}
    \caption{Generalized model architecture for DLRMs. The shaded structure (with orthogonal green stripes) is the baseline DLRM and the yellow box to its right represents the alternative proposed TT-Rec design that would replace the baseline in the stack.
    }
    \label{fig:tt-dlrm-arch}
\end{figure}

The sparse embedding tables pose infrastructure challenges from both the perspectives of storage and bandwidth requirement. There are tens of millions of rows in an embedding table, where the number of rows are growing exponentially for future recommendation models, resulting in memory requirement into the TB scale~\cite{zhao2020distributed}. Furthermore, the EMB lookup operation gathers multiple embedding vectors simultaneously across the tables, making the execution memory bandwidth bound. To address the ever-increasing memory capacity and bandwidth challenges, this work examines a fundamentally different approach---instead of gathering dense embedding vectors from EMBs that store information in the latent space representation, we seek methods to replace large EMBs with a sequence of small matrix products using tensor train decomposition.     


\paragraph{Tensor-train decompositions.}
Similar to matrix decomposition, such as Singular Value Decomposition (SVD)~\cite{xue2013restructuring} and Principal Component Analysis (PCA)~\cite{wold1987principal}, Tensor-Train decomposition is an approach to matrix decomposition by decomposing tensor representation of multidimensional data into product of smaller tensors. 
Tensor-Train (TT) decomposition is a simple and robust method for model compression~\cite{OSELEDETS2011} and has been studied extensively for deep learning application domains, such as computer vision~\cite{Yang2017} and natural language understanding~\cite{Rusu2017}. However, such method has not been investigated for the deep learning recommendation space; thus, its potential remains unknown. Here, we describe the fundamental principle of TT-decomposition and its application to recommendation embedding. 
%

Assume $\mathcal{A} \in \mathbb{R}^{I_1 \times I_2 \times \dots I_d}$ is a $d-$dimensional tensor, where $I_k$ is the size of dimension $k$. $\mathcal{A}$ can be decomposed as
\begin{equation}
    \mathcal{A}(i_1, i_2, \dots, i_d) = \mathcal{G}_1(:, i_1, :)\mathcal{G}_2(:, i_2, :)\dots \mathcal{G}_d(:, i_d, :),
\end{equation}
where $\mathcal{G}_k \in \mathbb{R}^{R_{k-1}\times I_k \times R_k}$, and $R_0 = R_k = 1$ to keep product of the sequence of tensors a scalar. The sequence $\{{R_k}\}_{k=0}^d$ is referred as to \textbf{TT-ranks}, and each 3-dimension tensor $\mathcal{G}_k$ is called a \textbf{TT-core}.

The TT decomposition can also be generalized to compress a matrix $W \in \mathbb{R}^{M \times N}$. We assume that $M$ and $N$ can be factorized into sequences of integers, i.e.,
    $M = \prod_{i = 1}^d m_k,$ and $N = \prod_{i = 1}^d n_k$.
Correspondingly, we reshape the matrix $W$ as a $2d$-dimensional tensor $\mathcal{W} \in \mathbb{R}^{(m_1\times n_1)\times(m_2\times n_2)\times \dots \times (m_d\times n_d)}$, where
\begin{align}
\begin{split}
    \mathcal{W}&((i_1, j_1),(i_2, j_2), \dots, (i_d, j_d)) \\
    &= \mathcal{G}_1(:, i_1, j_1, :) \mathcal{G}_2(:, i_2, j_2, :) \dots \mathcal{G}_d(:, i_d, j_d, :) 
\end{split}
\label{eqn:tt-matrix}
\end{align}
and each 4-d tensor $\mathcal{G}_k \in \mathbb{R}^{R_{k-1}\times m_k \times n_k \times R_k}$, $R_0 = R_d = 1$.
Let $R, m$, and $n$ be the maximal $r_k, m_k$ and $n_k$ respectively for $k = 1, \dots, d$. TT format reduces the space for storing the matrix from $O(MN)$ to $O(dR^2 \max(m,n)^2)$. Table~\ref{tab:kaggle-dim-tt} in Section~\ref{results} illustrates the detailed embedding table sizes and their corresponding compressed dimensions.
\comment{
\textcolor{red}{refer to the table. remove this paragraph}Taking one large embedding table in DLRM for instance, to decompose the matrix $W$ of size $10^7 \times 64$ in a 3-d manner, one can factorize the two dimensions by $10^7 = 200 * 200 * 250$, and $64 = 4*4*4$. The corresponding sizes of the 3 tensor cores are $(1, 200, 4, R)$, $(R, 200, 4, R)$, and $(R, 250, 4, 1)$, which reduces the number of parameters from $6.4*10^8$ to $800R^2 + 1800R$. Taking $R = 16$, the matrix is compressed by $2738\times$; and taking $R = 64$, the matrix is compressed by $188 \times$. We will illustrate in the following sections that such selections are sufficient to preserve the performance of DLRM.}

\label{preliminary}

\section{Proposed Design of TT-Rec}
In this section, we present the proposed design, called \textit{TT-Rec}. TT-Rec customizes the TT-decomposition method to compress embedding tables in deep learning recommendation models (\S~\ref{sec:method-compress}).
An overview of our design, where the large embedding tables are replaced with TT-Rec, is shown in Figure~\ref{fig:tt-dlrm-arch}.
In order to compensate the accuracy loss from replacing embedding vectors with vector-matrix multiplications, we introduce a new way to initialize the element distribution for the TT-cores (\S~\ref{sec:method-init}).
This is an important step, since the distribution of the weights resulting from multiplication of the TT-cores are skewed from TT-cores’ initial distribution due to the product operation.



\subsection{Customizing TT-decomposition for Embedding Table Compression}
\label{sec:method-compress}

Each embedding lookup can be interpreted as a one-hot vector matrix multiplication $w_i^T = e_i^T W$, where $e_i$ is a vector with $i$-th position to be 1, and 0 anywhere else. In more complex scenarios, an embedding lookup represents a weighted combination of multiple items $w = \sum_k i_k^T W$. We compress the embedding table $W$ as in Equation (\ref{eqn:tt-matrix}),
and hence embedding lookup operation of row $i = \sum_{i=1}^d i_k \prod_{j = i+1}^d I_j$ is be represented the following:
\begin{equation}
    w_i = \mathcal{G}_1(i_1, :, :) \mathcal{G}_2(:, i_2, :, :) \dots \mathcal{G}_d(:, i_d, :) 
    \label{eqn:tt-embedding}
\end{equation}
Let $w_i^{(k)} \in \mathbf{R}^{\prod_1^{k-1} n_i \times n_kR_k}$ be the partial product of the first $k$ TT-cores in Equation~\ref{eqn:tt-embedding}. The tensor multiplication of $w_i^{(k)}\mathcal{G}_{k+1}(:, i_k, :, :)$ can be unfolded and formulated as a matrix-matrix multiplication where $w_i^{(k)} \in \mathbf{R}^{\prod_1^k n_i \times R_k}$ and $\mathcal{G}_{k+1}(:, i_k, :, :) \in \mathbf{R}^{R_k\times n_{k+1}R_{k+1}}$.

In uncompressed models, storing and updating an embedding table of size $M \times N$ requires for $O(MN)$ space, while in TT-Rec, we propose to only learn the gradient of the loss function $L$ with respect to the TT-cores through backward propagation (Equation~\ref{eqn:chain-rule}).

%
\begin{align}
    &\frac{\partial L}{\partial \mathcal{G}_k(:, i_k, :)} =
        \begin{cases}
            w_i^T\frac{\partial L}{\partial y}, \text{ if $k = d$}\\
            (w_i^{(k)})^T\frac{\partial L}{\partial w_i^{(k+1)}},  \text{ if $1< k < d$}
        \end{cases}\\
    & \frac{\partial L}{\partial w_i^{(k)}} =
        \begin{cases}
            \frac{\partial L}{\partial y} \mathcal{G}_d^T(:, i_d, :), , \text{ if $k = d$}\\
            \frac{\partial L}{\partial w_i^{(k+1)}} \mathcal{G}_{k}^T(:, i_k, :, :),  \text{ if $1< k < d$}
        \end{cases}
     \label{eqn:chain-rule}  
\end{align}

\textit{Not all sparse features are equally important} is a key observation used in the design of TT-Rec. Data samples in industry-scale recommendation use cases often follow a Power or Zipfian distribution~\cite{wu2020developing}. A small subset of sparse features capture a significant portion of training samples that index to the set of the features in the embedding tables. 
This observation is particularly important to reduce the data movement overhead when GPUs are employed as AI training accelerators. 

One of the performance optimization potential unlocked by TT-Rec is that the collection of DLRMs that require memory capacities larger than that of training accelerators can now be accelerated with the accelerators. In addition to optimizing TT-Rec's performance using GPUs, we introduce a caching scheme to retain the frequently-accessed embedding vectors/rows in the EMBs. The cache enables TT-Rec to exploit the aforementioned temporal locality by storing most frequently-accessed embedding rows in the uncompressed format. By doing so, TT-Rec minimizes the need of computation. The detail of the performance optimization implementation are described in detail later in Section~\ref{performance}. 

\begin{figure}[t]%
       \centering
       \subfloat{\includegraphics[width=0.45\linewidth]{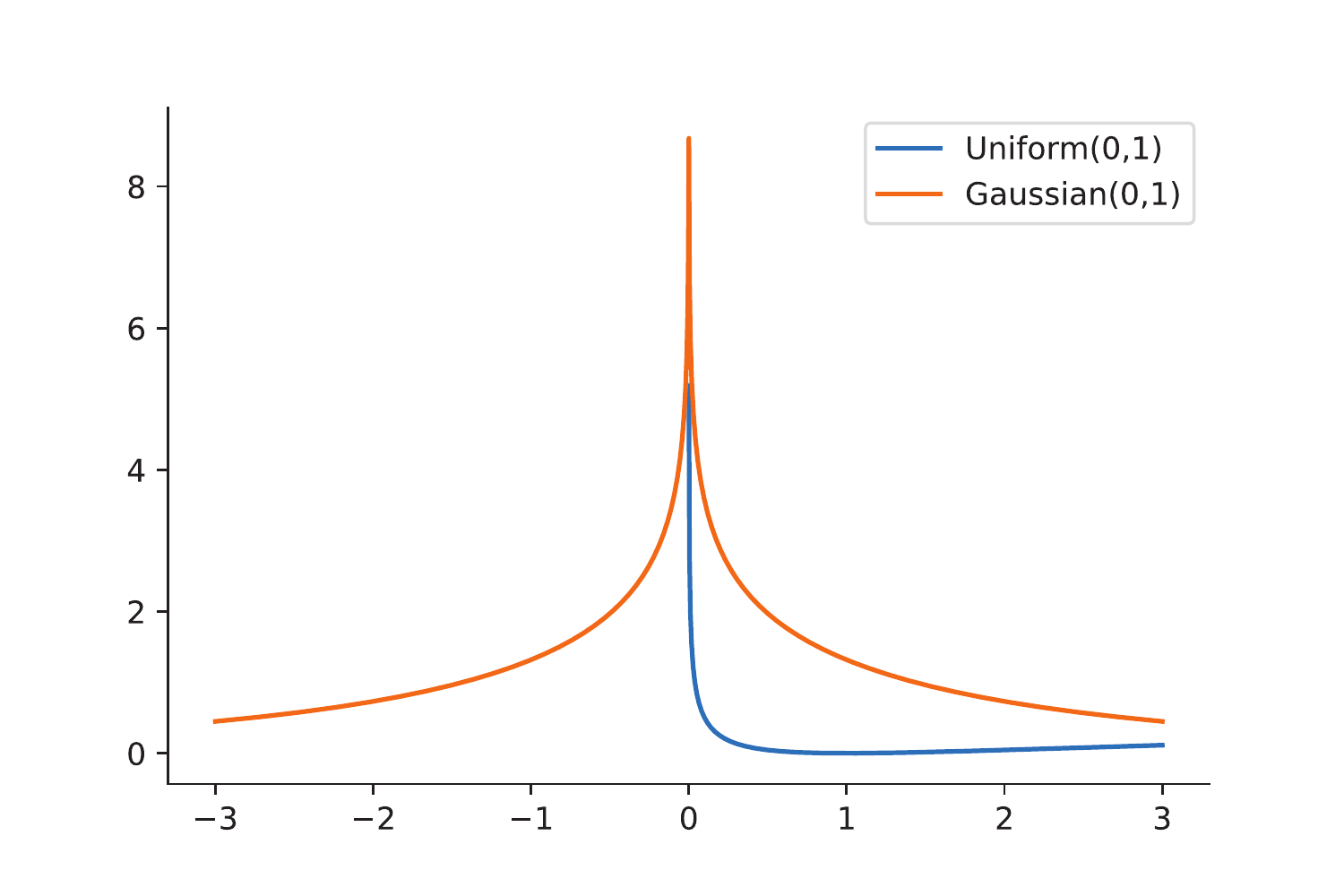} }%
       \subfloat{\includegraphics[width=0.45\linewidth]{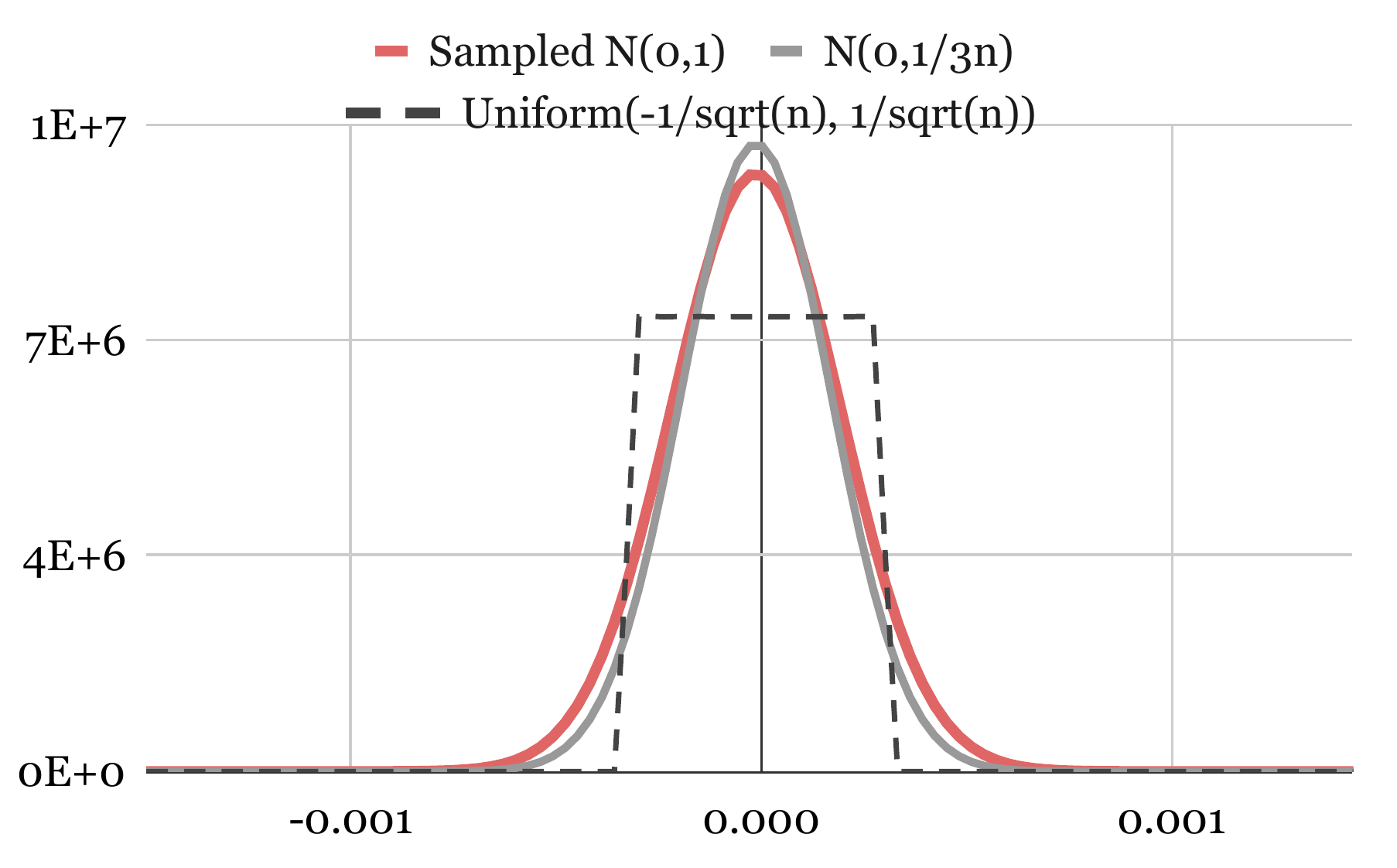} }%
       \vspace{-4mm}
       \caption{The probability density function (PDF) of product of three independent and identically distributed (i.i.d.) random variables (left) from Uniform(0,1) and from $\mathcal{N}$(0,1); and from (right) sampled Gaussian distribution (red), comparing to $\mathcal{N}(0,1)$ (grey) and uniform distribution.
        }%
       \label{fig:rv-pdf}%
       \vspace{-0.25cm}
   \end{figure}
   
\subsection{Weight Initialization}
\label{sec:method-init}

Replacing an embedding vector lookup with a sequence of tensor computation to approximate the original vector may introduce an accuracy loss in TT-Rec. To compensate the accuracy loss, we introduce a new way to initialize the weights of the TT-cores.
Typically, larger TT-ranks provide lower compression ratios while achieving model accuracy closer to that of the baseline. However, when increasing the TT-rank value in TT-Rec, we find that 
the corresponding accuracy improvements saturate quickly despite the decreasing compression ratios, as we show later in detail in Section~\ref{results}.
This led us to investigate the initialization behavior of the uncompressed baseline and TT-Rec---the initial distribution of the TT-cores can significantly influence the model quality.
 
Since the TT decomposition approximates the full tensor, we hope to find the best configuration of the uncompressed model and have TT-Rec approximate the same configuration. For the uncompressed DLRMs, 
uniform random distribution usually outperforms normal distribution.
For validation, we initialize DLRMs with different forms of Gaussian distribution.
Then, we determine the
correlation between the model accuracy and the distance between the Gaussian and uniform distribution.

To approximate a uniform distribution on $[a,b]$ by a Gaussian distribution $\mathcal{N}(\mu, \sigma^2)$, we want to minimize the KL-divergence
\begin{equation*}
    \mathcal{D}_{KL}(P|Q) = -\int_{-\infty}^{\infty}P(x) \ln \frac{P(x)}{Q(x)} dx
\end{equation*}
where $P(x) = \frac{1}{b-a}$ if $x\in [a, b]$, otherwise 0; and $Q(x) = \frac{1}{\sqrt{2\pi \sigma^2}}e^{\frac{(x-\mu)^2}{2\sigma^2}}$.
\comment{
\begin{align}
    L(\mu, \sigma) &= -\int_{a}^{b}\frac{1}{b-a}  \ln(\frac{1}{b-a}/(\frac{1}{\sqrt{2\pi \sigma^2}}e^{\frac{(x-\mu)^2}{2\sigma^2}} ))dx\\
    &= \ln(b-a) - \frac{1}{2}\ln(2\pi\sigma^2) \\
    &- \frac{1/3(b^3-a^3) - \mu(b^2-a^2) + \mu^2(b-a)}{2\sigma^2(b-a)}
\end{align}
To minimize the loss, we compute the partial derivatives
\begin{align}
    \frac{\partial L}{\partial \mu} &= \frac{a+b}{2\sigma^2} - \frac{2\mu}{2\sigma^2}\\
    \frac{\partial L}{\partial \sigma} &= \frac{1}{\sigma} + \frac{1/3(a^2 + ab + b^2)-1/4(a+b)^2}{\sigma^3}
\end{align}
Solving the equations, we have
}
Minimizing the KL-divergence for a given uniform distribution using the first-order approach results in
\begin{equation*}
    \mu = \frac{a+b}{2}, \quad \sigma^2 = \frac{(b-a)^2}{12}.
\end{equation*}
Therefore, to best approximate the uniform distribution used in the DLRM (Uniform$(\frac{-1}{\sqrt{n}}, \frac{1}{\sqrt{n}})$), we should adopt $\mathcal{N}(0, \frac{1}{3n})$ for initialization. Table~\ref{tab:kl-accu} shows the accuracy of the DLRMs that are initialized with various Gaussian forms, where accuracy gap is proportional to the KL-divergence between the Gaussian and uniform distribution.

\begin{table}[]
\vspace{-0.25cm}
\vspace{2pt}
\begin{tabular}{c|c|c}
\hline
\textbf{Distribution} & \textbf{KL-divergence} & \textbf{Accuracy} \\ \hline
uniform($(\frac{-1}{\sqrt{n}}, \frac{1}{\sqrt{n}})$)  & 0             & 79.263 \%       \\ \hline
$\mathcal{N}(0,1)$           & $ c -\frac{1}{6n}$            & 78.123 \%      \\
$\mathcal{N}(0,1/2)$           & $c - \frac{1}{3n} + 0.34$             & 78.371 \%        \\
$\mathcal{N}(0,1/8)$         & $c - \frac{4}{3n} + 1.4$             & 78.823  \%       \\
$\mathcal{N}(0,1/3n)$   & $-0.17$            & 79.256 \%      \\
$\mathcal{N}(0,1/9n^2)$        & $\frac{1}{2}\ln \frac{18}{\pi n} -1.5n$             & 79.220 \%  \\ \hline      
\end{tabular}
\caption{The accuracy of uncompressed DLRM with embedding tables initialized from various Gaussian distributions, $c = 0.5\ln \frac{2}{\pi n} < -7$, compared with uniform distribution.
}
\vspace{-5mm}
\label{tab:kl-accu}
\end{table}
\comment{
\begin{figure*}
    \centering
    \includegraphics[width=\textwidth]{Figures/dlrm-full-dist.png}
    \vspace{-5mm}
    \caption{Left: product of 3 random variables i.i.d. from uniform(0,1). Right: product of 3 random variables i.i.d. from normal(0,1).}
    \label{fig:dlrm-dist}
\end{figure*}
}
For TT-Rec, we want the product of TT cores to either approximate Uniform$(\frac{-1}{\sqrt{n}}, \frac{1}{\sqrt{n}})$ or $\mathcal{N}(0, \frac{1}{3n})$ during initialization. Initializing the TT-cores is challenging since the distribution of product of random variables is non-trivial. In practice, both uniform and normal distributions can be used to reasonably initialize TT-cores~\cite{hrinchuk2020tensorized}. However, the TT product of these distributions do not serve as an appropriate approximation to uniform distribution as Figure~\ref{fig:rv-pdf} (left) shows.

To better approximate the best Gaussian distribution shown in Table~\ref{tab:kl-accu}, we reduce the amount of values close to zero
in each core by sampling the Gaussian distribution.
Our sampling method is shown in Algorithm~\ref{alg:sampled_gaussian} in Appendix~\ref{sec:appendix_algorithm}.
We validate the product of the TT-cores with Algorithm~\ref{alg:sampled_gaussian} in Appendix~\ref{sec:appendix_algorithm}. In Figure~\ref{fig:rv-pdf}-(right), we compare the product with $\mathcal{N}(0, 1/3n)$, which serves as the best approximation to Uniform$(-1/\sqrt{n}, 1/\sqrt{n})$. We will show in Section~\ref{sec:result-accuracy} that our algorithm achieves the highest accuracy in training.

\comment{
Let 
\[
\mathcal{A}(i_1, i_2, i_3) = 
\mathcal{G}_1(i_1, :)\mathcal{G}_2(:, i_2, :)\mathcal{G}_3(:, i_3)
\]
be the TT decomposition of $\mathcal{A}$, and all TT-core entries are i.i.d random. Let $Z = X_1X_2X_3$, and the PDF of elements in $\mathcal{A}$ can be computed by scaling the PDF of $Z$ by a constant factor.

\noindent
\textbf{Case 1. Product of Uniform Random Variables:} 
We use a 3-$d$ TT decomposition to illustrate the situation. Let 
\[
\mathcal{A}(i_1, i_2, i_3) = 
\mathcal{G}_1(i_1, :)\mathcal{G}_2(:, i_2, :)\mathcal{G}_3(:, i_3)
\]
be the TT decomposition of $\mathcal{A}$, and all TT-core entries are random variables i.i.d. from uniform distribution Uniform$(0,1)$. For simplicity, we omit the summations and let $Z = \sum_1^R\sum_1^R\sum_1^R X_1X_2X_3$. 

For simplicity we assume $X_i'$s are i.i.d from Uniform(0,1), since the analysis for other uniform distributions can be conducted in a similar manner. The CDF of Z is
\begin{align*}
    F(z) &= 
 - \int_{x = 0}^z \log x dx - \int_{x = z}^1 z/x \log x dx
\end{align*}
Hence the density of $Z$ is
$
f(z) = 0.5(\log z)^2, 0< z <=1. 
$

\noindent
\textbf{Case 2. Product of Gaussian Random Variables:}
If the entries of TT-cores are i.i.d from $\mathcal{N}(0,1)$, the product of them can be interpreted by the Meijer G-function~\cite{product-gaussian}. In Figure~\ref{fig:rv-pdf}-(a), the plot on shows the density of the product of $Z$, each from Uniform(0,1) or $\mathcal{N}$(0,1). Both distributions do not approximate uniform appropriately, while Gaussian initialization for TT-cores achieve higher accuracy than uniform in all of our experiments.
}

\label{method}

\section{Performance Optimizations}
Replacing embedding vectors with TT-core multiplications trades off memory with computation. Depending on the embedding vector lookup patterns and the underlying system architectures, training time overhead from the computations can vary. In order to mitigate the associated performance overhead, we introduce a cache structure as part of TT-Rec to leverage the observation -- a small subset of features (rows in embedding tables) constitute most of the embedding table accesses. Thus, here, we describe the execution flow for deep learning recommendation model training: forward propagation and backward propagation implementations in Section~\ref{sec:training-implementation}. Then, we introduce TT-Rec's cache design to improve the training time performance in Section~\ref{sec:training-cache}. 

\begin{figure}[t]
    \centering
    \includegraphics[width=0.9\columnwidth]{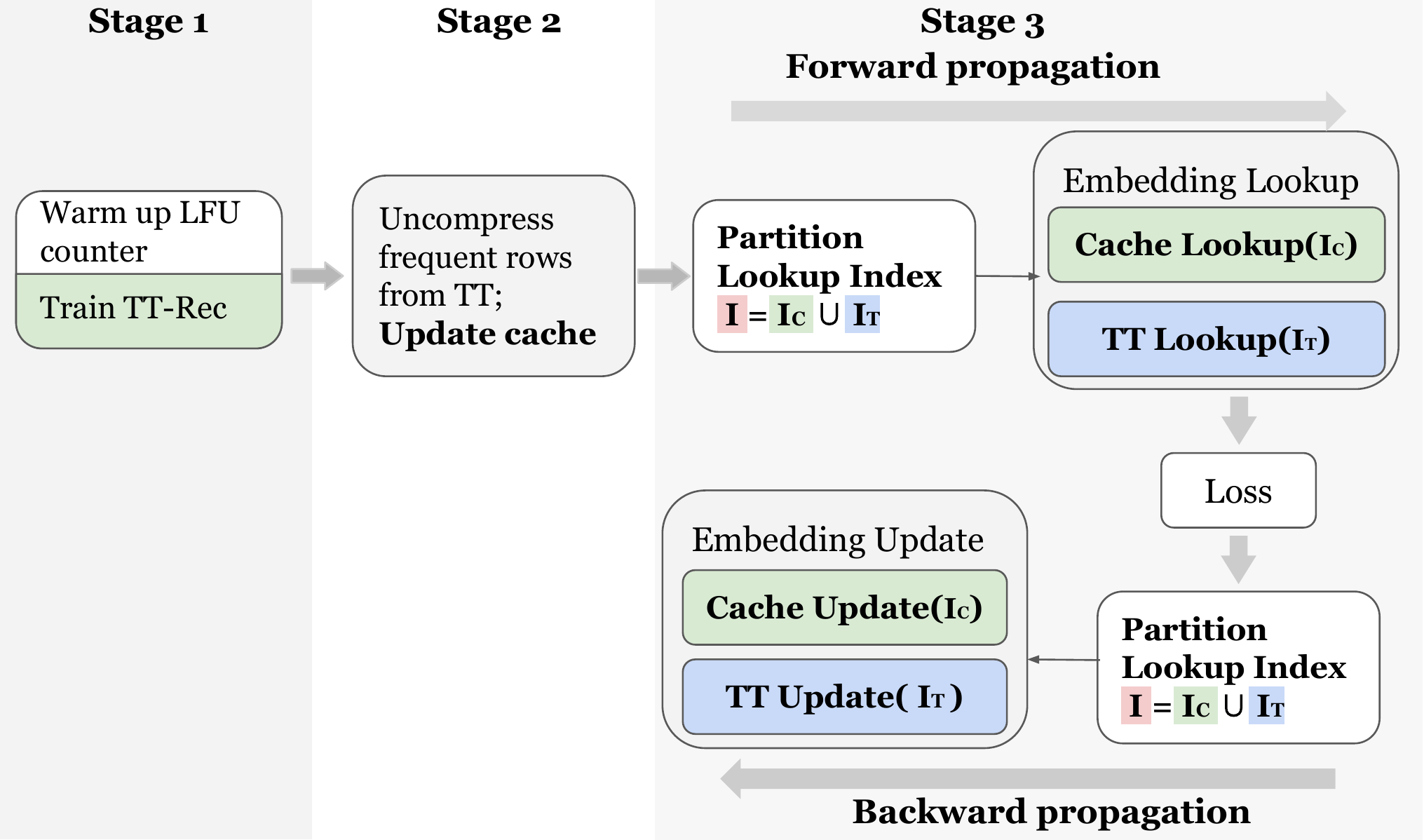}
    \caption{Multi-stage training process with caching.}
    \label{fig:lfu-multi-stage}
    \vspace{-0.25cm}
\end{figure}

\subsection{TT-Rec Implementation}
\label{sec:training-implementation}

The embedding layer takes row indices as input in Compressed Sparse Row (CSR) format. The input is translated into multiple embedding bags, in which the embedding vectors belonging to an embedding bag are pooled together by summation or average. 
Let $indices[]$ be an integer array of length $n$, where $indices[i] < Index\_Range$ for $0\le i < n$, and $offsets[]$ be an integer array of length $m \le n$ to specify the starting index in $indices[]$ of each embedding bag. Specifically, 
to compute for $j$th embedding bag, the algorithm will summarize or average the rows $\{W(indices[k], :)\}$ for $offsets[j] \le k < offsets[j+1]$. In TT-Rec, this is computed as a sequence of small matrix-matrix products.

\begin{align}
    &output_i = \sum_{k = offset[i]}^{offset[i+1]} \alpha_{indices[k]}w_{indices[k]}\\
    &= \sum_{k = offset[i]}^{offset[i+1]} \alpha_{indices[k]}\mathcal{G}_1(i^k_1, :, :)  \dots \mathcal{G}_d(:, i^k_d, :) 
    \label{eqn:tt-embedding-bag}
\end{align}
, where $i_j^k$ indexes the slice in $\mathcal{G}_j$ that is used for computing vector $W(indices[k], :)$, and $\alpha_i$ is the per sample weight associated with each embedding vector.

To obtain a high-performance implementation of TT-EmbeddingBag, we use a batched GEMM implementation from cuBLAS~\cite{batchedgemm}. The algorithm computes for a batch of $B$ embedding vectors and stores the intermediate results in $tr_i$, $0\le i \le d-2$. Given a batch of queried embedding vectors, the algorithm reduces the batch to embedding bags in parallel. Algorithm~\ref{alg:tt-emb} of Appendix~\ref{sec:appendix_algorithm} shows the pseudocode of the batched embedding kernel for the 3-dimensional TT decomposition. This algorithm can be generalized to any arbitrary TT dimension by extending lines 5-10 to set up the pointers to intermediate TT results.

The backward propagation algorithm for the 3-dimensional TT compression is illustrated by Algorithm~\ref{alg:tt-emb-back} of Appendix~\ref{sec:appendix_algorithm}. Based on the chain rule as described in Equation~\ref{eqn:chain-rule}, we compute the gradient w.r.t. the embedding bags, and accumulate the gradients into each TT-core.

\subsection{A Least-Frequently-Used Cache for TT-Rec}
\label{sec:training-cache}
While TT-Rec reduces the memory requirement of embedding tables, this method introduces latency through additional computations from (1) reconstructing values of the queried embedding rows from the TT format in the forward propagation, (2) determining the gradient for each TT-core in the backward propagation, and (3) recomputing the intermediate GEMM results for the gradient computation. 


\begin{table*}[t]
\resizebox{0.9\textwidth}{!}{%
\begin{tabular}{@{}c|c|c|c|c|c|c|c|ccc@{}}
\toprule
\multicolumn{2}{c|}{\begin{tabular}[c]{@{}c@{}}Emb. Table Dimensions\end{tabular}} & \multicolumn{3}{c|}{TT-Core Shapes}                  & \multicolumn{3}{c|}{\# of TT Parameters} & \multicolumn{3}{c}{Memory Reduction}                                           \\ \midrule
\# Rows                                  & Emb. Dim.                                 & $\mathcal{G}_1$ & $\mathcal{G}_2$ & $\mathcal{G}_3$ & $R=16$    & $R=32$    & $R=64$   & \multicolumn{1}{c|}{$R=16$} & \multicolumn{1}{c|}{$R=32$} & $R=64$ \\ \midrule
10131227                                     & 16                                        & (1, 200, 2, $R$)  & ($R$, 220, 2, $R$)  & ($R$, 250, 4, 1)  & 135040        & 495360        & 1891840      & \multicolumn{1}{c|}{1200} & \multicolumn{1}{c|}{327}  & 86   \\
8351593                                      & 16                                        & (1, 200, 2, $R$)  & ($R$, 200, 2, $R$)  & ($R$, 209, 4, 1)  & 122176        & 449152        & 1717504      & \multicolumn{1}{c|}{1094} & \multicolumn{1}{c|}{297}  & 78   \\
7046547                                      & 16                                        & (1, 200, 2, $R$)  & ($R$, 200, 2, $R$)  & ($R$, 200, 4, 1)  & 121600        & 448000        & 1715200      & \multicolumn{1}{c|}{927}  & \multicolumn{1}{c|}{252}  & 66   \\
5461306                                      & 16                                        & (1, 166, 2, $R$)  & ($R$, 175, 2, $R$)  & ($R$, 188, 4, 1)  & 106944        & 393088        & 1502976      & \multicolumn{1}{c|}{817}  & \multicolumn{1}{c|}{222}  & 58   \\
2202608                                      & 16                                        & (1, 125, 2, $R$)  & ($R$, 130, 2, $R$)  & ($R$, 136, 4, 1)  & 79264         & 291648        & 1115776      & \multicolumn{1}{c|}{445}  & \multicolumn{1}{c|}{121}  & 32   \\
286181                                       & 16                                        & (1, 53, 2, $R$)   & ($R$, 72, 2, $R$)   & ($R$, 75, 4, 1)   & 43360         & 160448        & 615808       & \multicolumn{1}{c|}{106}  & \multicolumn{1}{c|}{28}   & 7    \\
142572                                       & 16                                        & (1, 50, 2, $R$)   & ($R$, 52, 2, $R$)   & ($R$, 55, 4, 1)   & 31744         & 116736        & 446464       & \multicolumn{1}{c|}{72}   & \multicolumn{1}{c|}{19}   & 5    \\ \bottomrule
\end{tabular}%
}
\caption{The original dimensions of Kaggle's 7 largest embedding tables in DLRM and their respective TT decomposition parameters.}
\label{tab:kaggle-dim-tt}
\vspace{-0.25cm}
\end{table*}

Recomputation of the intermediate results, in Algorithm~\ref{alg:tt-emb-back} (line 3) of Appendix~\ref{sec:appendix_algorithm}, can be eliminated by storing tensors from the forward pass. This reduces the latency but comes with slightly increased memory footprint and memory allocation time. The dominating source of performance optimization potential, however, comes from the aforementioned distribution of sparse features in the training samples. 
The row access frequency follows the Power Law distribution in the largest embedding tables, i.e. a few embedding vectors are recurrently accessed throughout training.
%


To consider this unique characteristics of recommendation data samples, we design and implement a software cache structure to store the frequently-used embedding vectors on the GPU. The cache stores an uncompressed copy of the frequently accessed embedding vectors.
As Section~\ref{sec:cache_performance_analysis} later shows, we perform performance sensitive analysis over a wide range of cache sizes and empirically determine that devoting 0.01\% of the respective embedding table as TT-Rec cache is sufficient from both model accuracy's and training time's perspectives.

Given the sequence of embedding bags, the indices are first partitioned into two different groups: 
$cached\_indices$ and $tt\_indices$. The cached embedding rows are directly fetched from the cache whereas the embedding vectors from the $tt\_indices$ group are computed, following the forward propagation
operations described in Algorithm~\ref{alg:tt-emb-back}.
Then, during the backward propagation, the cached, uncompressed vectors can be simply updated with $W' = W + \frac{\partial L}{\partial W}$, while the non-cached vectors
are updated to TT-cores as described in Algorithm~\ref{alg:tt-emb}.
%
In this way, the weights of the two index sets are learned separately by the cache and the TT cores.

In order to offset the cache population overheads,
TT-Rec adopts a semi-dynamic cache, where the frequently-accessed
embedding vectors are loaded into cache
only every 100s to 1000s of iterations,
initialized from TT cores.
In order to track the frequencies of the all the existing indices,
an open addressing hash table is used.
On the other hand,
the learned weights in the cache would be discarded when an eviction happens. In practice, this strategy does not affect training accuracy as the evicted cache lines are not accessed frequently and therefore contribute less to the overall model.
We chose this strategy
as decomposing the evicted vectors and updating the decomposed
parameters with the existing TT cores are equivalent to dynamically
tracking TT decomposition for a streaming matrix, which is a challenging algebraic problem itself.

Figure~\ref{fig:lfu-multi-stage} summarizes the multi-stage training process. The model training starts with the TT embedding tables only. The first few iterations (e.g. 10\% training samples) are used to warm up the cache state. The most frequently accessed embedding vectors will be stored in the cache as uncompressed. Depending on the phase behavior, one might consider updating the cache and repeat the warm up process periodically. Based on our empirical observation, the set of the most-frequently-accessed vectors are stable over window, indicating little periodic warm-up need. 

\label{performance}

\section{Experimental Setup}

\paragraph{Deep Learning Recommendation Model Parameters and Data Sets:}

We implement the proposed TT-Rec design over the open-source MLPerf reference implementation of the DLRM recommendation model architecture (MLPerf-DLRM)~\cite{Naumov2019,mlperf-training,mattson2020mlperf-ieeemicro,mlperf}.
We train TT-Rec with the Criteo Kaggle Display Advertising Challenge Dataset (Kaggle)\footnote{\url{https://labs.criteo.com/2014/02/kaggle-display-advertising-challenge-dataset/}} and Criteo Terabyte Click Logs (Terabyte)\footnote{\url{https://labs.criteo.com/2013/12/download-terabyte-click-logs/}}. In both the datasets, each data sample consists of 13 numerical features and 26 categorical features, and a binary label. Kaggle contains 7 days of ads click data, whereas Terabyte contains 24 days of click data (4.3 billion records). The 26 categorical features are interpreted into 26 embedding tables in TT-Rec, where each row in the embedding table corresponds to an element in that category.
Figure~\ref{fig:dlrm_reduction}-Baseline bars show the size and composition the embedding tables in the two datasets. Table~\ref{tab:kaggle-dim-tt} summarizes the dimensions of the 7 largest embedding tables when training MLPerf-DLRM with Kaggle and the respective TT decomposition parameters.


To study the effectiveness of TT-Rec, we adopt the same hyperparamters as specified in the MLPerf-DLRM reference implementation, including the embedding dimensions, MLP dimensions, learning rate, and batch size. Both datasets are trained with the SGD optimizer.
Note, as the dimension of the embedding increases from 64 to 512, the total memory requirement is over 96 GB, exceeding the latest GPU memory capacity. This is when TT-Rec shines. The uncompressed baseline has to run on CPUs or multiple GPUs via model parallelism (which requires extra all-to-all communication overheads) while TT-Rec enables recommendation training on GPUs with data parallelism.

For accuracy evaluation, we report the test accuracy (\%) as well as the BCE loss. We train TT-Rec for a single epoch using all the data samples in Kaggle. For Terabyte, we downsize the negative training samples by 0.875, as specified by the MLPerf-DLRM benchmark. 



All the evaluation results are obtained by training Kaggle dataset on an NVIDIA Tesla V100-SXM2 GPU with an Intel Xeon E5-2698 CPU. Terabyte dataset results are obtained on the same system but using eight CPUs with a single GPU due to its larger system memory requirement.

\label{experimental_setup}

\section{Experimental Results}

\begin{figure}[t]
    \centering
    \includegraphics[width=\columnwidth]{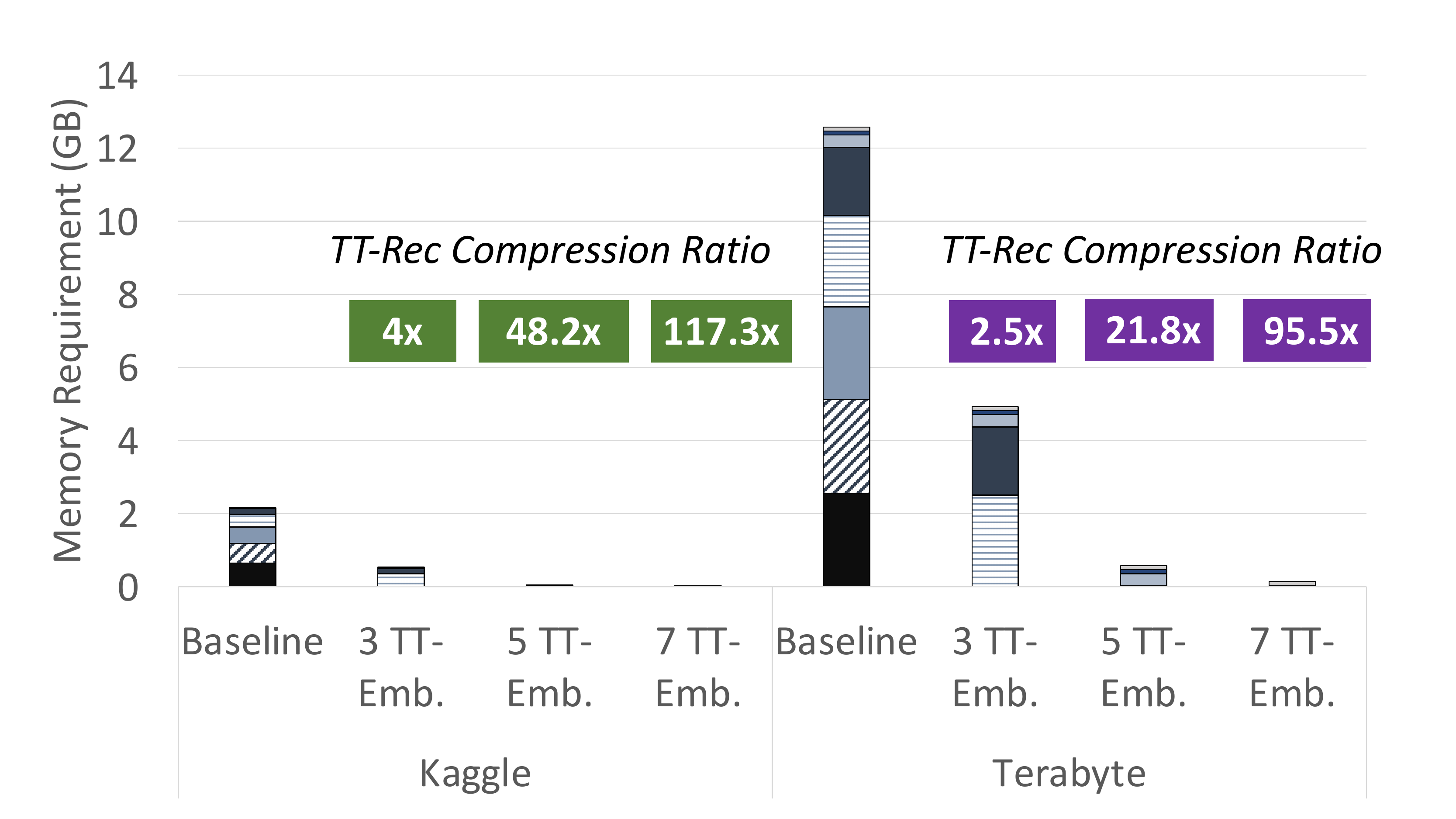}
    \vspace{-0.5cm}
    \caption{TT-Rec shows significant model size reduction when compressing different number of embeddings for Kaggle in green, and Terabyte in purple labels (TT-rank of 32).
    }
    \label{fig:dlrm_reduction}
    \vspace{-3mm}
\end{figure}

\begin{figure*}[bht]
     \centering
     \includegraphics[width=0.97\textwidth]{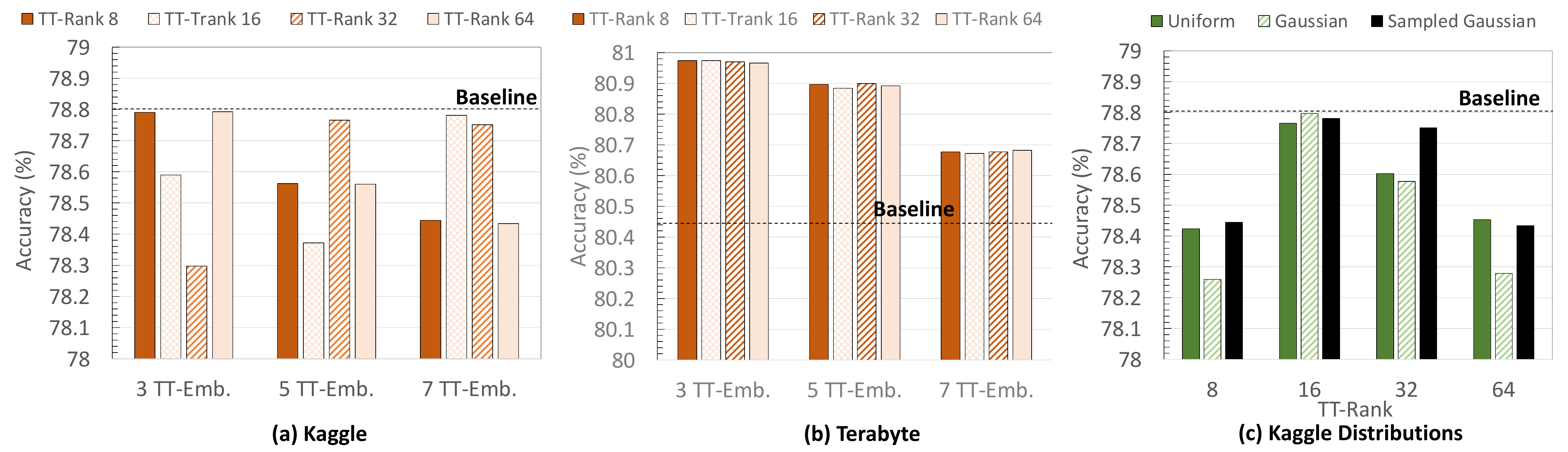}
        \vspace{-5mm}
        \caption{(a) and (b) Validation accuracy of TT-Rec with various tables compressed and TT ranks training with Kaggle and Terabyte respectively. (c) Validation accuracy of TT-Rec using different initialization techniques on Kaggle.}
        \label{fig:accuracy}
\end{figure*}
In this section, we present the evaluation results for the proposed TT-Rec design (\S\ref{experimental_setup}). We evaluate TT-Rec across the three important dimensions of DLRMs: memory capacity reduction (\S\ref{sec:capacity-reduction}), model quality (\S\ref{sec:result-accuracy}), and training time performance (\S\ref{sec:training-time-ttrec}). 
Then, \S\ref{sec:implementation-efficiency}, \S\ref{sec:cache_performance_analysis} and \S\ref{sec:tt-rec-pooling} provide a deeper performance analysis for the TT-Embedding kernel implementation, the TT-Rec cache design, and embedding-dominated DLRMs, respectively.


\paragraph{TT-Rec:} Overall, TT-Rec demonstrates to be an effective technique. 
For Kaggle, TT-Rec reduces the overall model size requirement by 117$\times$ from 2.16 GB to 18.36 MB. The significant capacity reduction comes with a relatively small amount of training time increase by 14.3\% while maintaining the accuracy as the baseline. 
Similarly, for Terabyte, the overall model size requirement is reduced by 112$\times$ from 12.57 to 0.11 GB.
This capacity reduction also comes with a relatively small amount of training time increase by 13.9\%. Finally, the model quality experiences negligible degradation. For example, for both Kaggle and Terabyte, using TT-Rec to train the five largest embedding tables in the TT-Emb. format using TT-rank of 32 leads to model accuracy loss within 0.03\%.




\subsection{Memory Capacity Requirement with TT-Rec}
\label{sec:capacity-reduction}

TT-Rec achieves significant compression ratios for the embedding tables of Terabyte and Kaggle, by as much as 327$\times$ and by an average of 181$\times$ (with TT-rank of 32). In the uncompressed baseline, the 7 largest tables constitutes 99\% of the model. For Kaggle, with TT-Rec, the memory requirement of the 7 embedding tables is reduced from 2.16 GB to only 18 MB, leading to 112$\times$ model size reduction. 

Figure~\ref{fig:dlrm_reduction} compares the memory capacity requirement between the baseline and TT-Rec (x-axis) across the 3, 5, and 7 largest embedding tables (y-axis). As illustrated in Figure~\ref{fig:dlrm_reduction}, the model size requirement also becomes significantly lower when TT-Rec trains the less number of the large embedding tables in the TT-Emb. format -- for TT-Emb. of 5 and 3, the overall model size is reduced by 48$\times$ and 4$\times$, respectively.
For Terabyte, TT-Rec achieves 2.6, 21.8, and 95.5$\times$ model size reduction for TT-Emb. of 3, 5 and 7, respectively.
%
This impressive memory capacity reduction unlocks industry-scale multi-GB/TB DLRMs that cannot be previously trained using commodity training accelerators, such as GPUs (40 GB of HBM2 in the latest NVIDIA A100 GPUs~\cite{nvidia:a100} or TPUs 16-32 GB of HBM~\cite{tpu}), to enjoy significantly higher throughput performance in state-of-the-art training accelerators.




\subsection{Model Accuracy with TT-Rec}
\label{sec:result-accuracy}


To achieve accuracy-neutral while still enjoying TT-Rec's memory reduction benefit, Figure~\ref{fig:accuracy}(a) and (b) compare the validation accuracy of the uncompressed baseline with that of TT-Rec for Kaggle and Terabyte, respectively. 
Figure~\ref{fig:accuracy}(a) shows that, when TT-Rec trains the largest 3, 5, and 7 embedding tables in the TT-Emb. format (x-axis), the optimal TT-rank to achieve a nearly accuracy-neural result varies, with the optimal rank of 8, 32, and 64, respectively.

Interestingly, for Terabyte, Figure~\ref{fig:accuracy}(b) shows that TT-Rec achieves higher validation accuracies (y-axis) across the board. 
As expected, with more embedding tables trained in the TT-Emb. format, TT-Rec brings significantly higher model size reduction at the expense of model accuracy degradation. Increasing the number of the large embedding tables trained in the TT-Emb. format from 3 to 7 improves the model size reduction from 2.6 to 95.5$\times$ while the validation accuracy degrades from 80.975\% to 80.682\%. Note, even though the validation accuracy is lowered, the model accuracy (TT-Emb. of 7) still outperforms the uncompressed baseline of 80.45\%. 


Using larger TT-ranks produces more accurate models at the expense of lower compression ratios. 
We notice that, although mathematically larger TT-ranks should produce more accurate approximations to the full tensor, increasing the rank does not always compensate the loss of accuracy. 
We believe that such accuracy loss is caused by the weight initialization distribution, as we describe next. 
\begin{figure}[t]
    \centering
    \vspace{-2.5mm}
    \includegraphics[width=0.9\columnwidth]{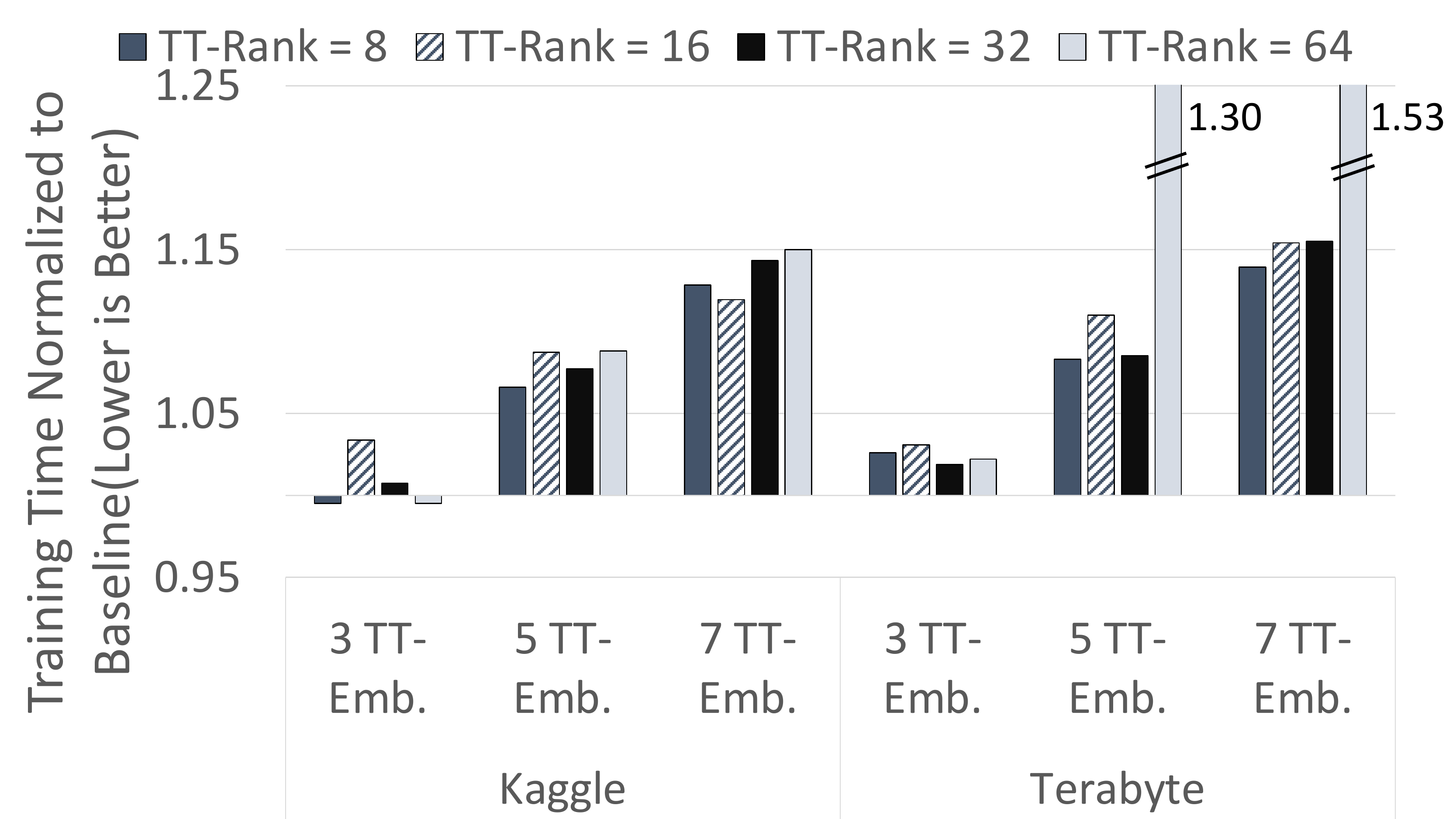}
    \caption{TT-Rec training time comparison across TT-ranks and the setting of TT-Emb.
    Baseline takes 12.14ms/iter on Kaggle, and 12.64ms/iter on Terabyte.}
    \vspace{-5mm}
    \label{fig:kaggle-dlrm-time}
\end{figure}

Figure~\ref{fig:accuracy}(c) presents the TT-Rec accuracy results using the different weight initialization strategies, described in \S\ref{sec:method-init}. Recall, the model accuracy difference between TT-Rec and the baseline strongly correlates with the distance between the distribution of the full matrix generated by the set of the tensor cores and the uniform distribution. Thus, the sampled Gaussian distribution is expected to outperform the other distributions as the full matrix generated by TT-Rec approximate $\mathcal{N}(0, \frac{1}{3n})$ the best. The accuracy results in Figure~\ref{fig:accuracy}(c) verify this expectation empirically.

\begin{figure}[t]
    \centering
    
    \includegraphics[width=.8\columnwidth]{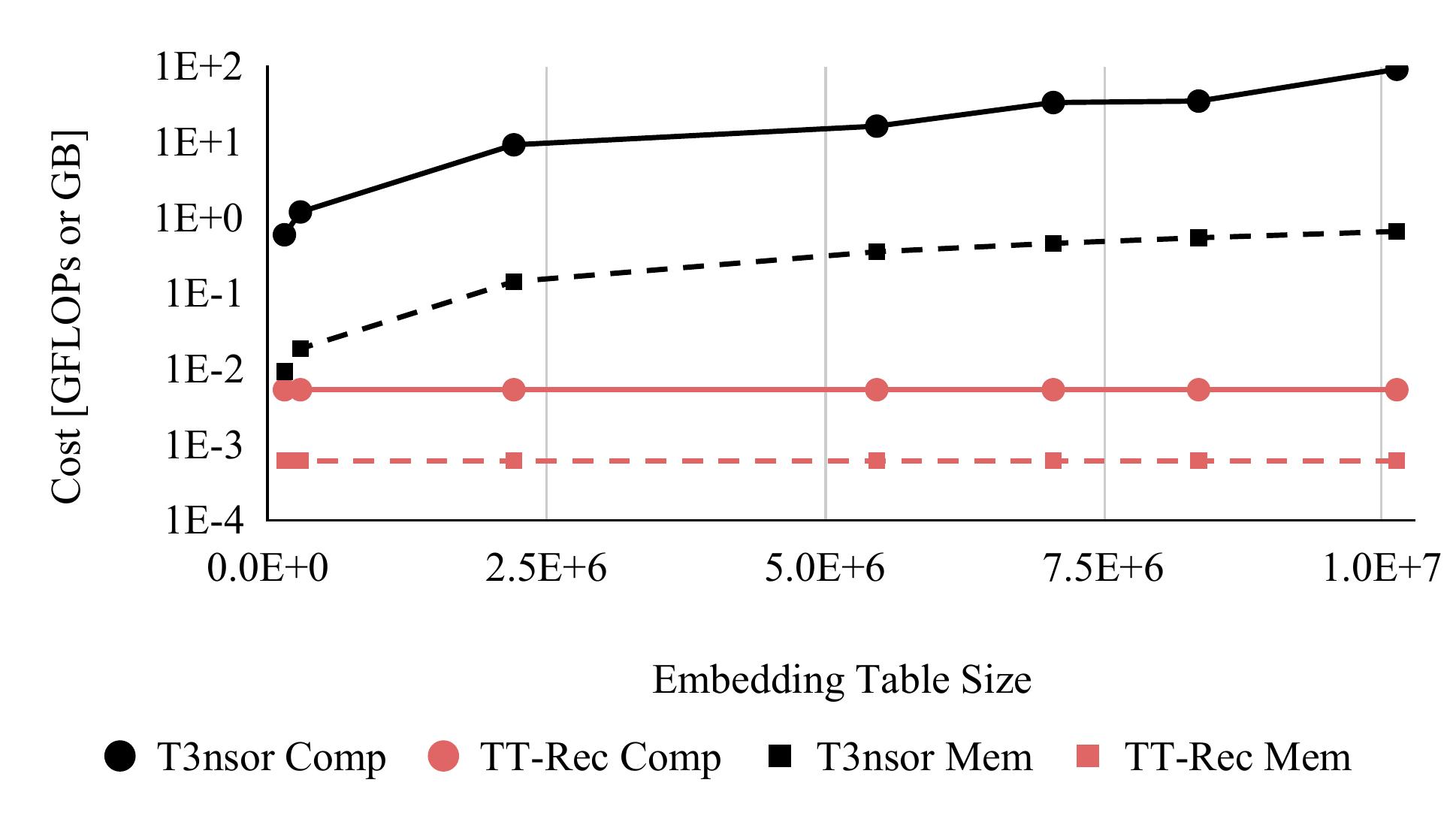}
    \vspace{-2.5mm}
    \caption{System resource requirement of T3nsor and TT-Rec.}
    \label{fig:tt-comp-cost}
    \vspace{-0.25cm}
\end{figure}

\subsection{Training Time Performance of TT-Rec}
\label{sec:training-time-ttrec}
To have a full picture of TT-Rec's potential, we present the training time performance of TT-Rec next.
Figure~\ref{fig:kaggle-dlrm-time} depicts the normalized training time of TT-Rec (y-axis), using TT-ranks of [8, 16, 32, 64] across the TT-Emb. settings of [3, 5, 7] (x-axis). 
Higher model size reduction ratios come with higher training time overheads. Increasing the number of the large embedding tables trained in the TT-Emb. format from 3 to 7 reduces the model sizes by 46.5 and 37.4$\times$ for Kaggle and Terabyte, respectively, while the training time using the optimal TT-rank increases by 12.5\% and 11.8\%, respectively. 
%
Depending on the importance of the three axes---\textit{memory capacity requirement}, \textit{model quality} and \textit{training time performance} for DLRM training---TT-Rec offers a flexible design space that can be navigated according to the desired optimization goal.

\subsection{TT-Embedding Kernel Implementation Efficiency}
\label{sec:implementation-efficiency}

\begin{figure}[t]
    \centering
    \vspace{-2.5mm}
    \includegraphics[width=0.8\columnwidth]{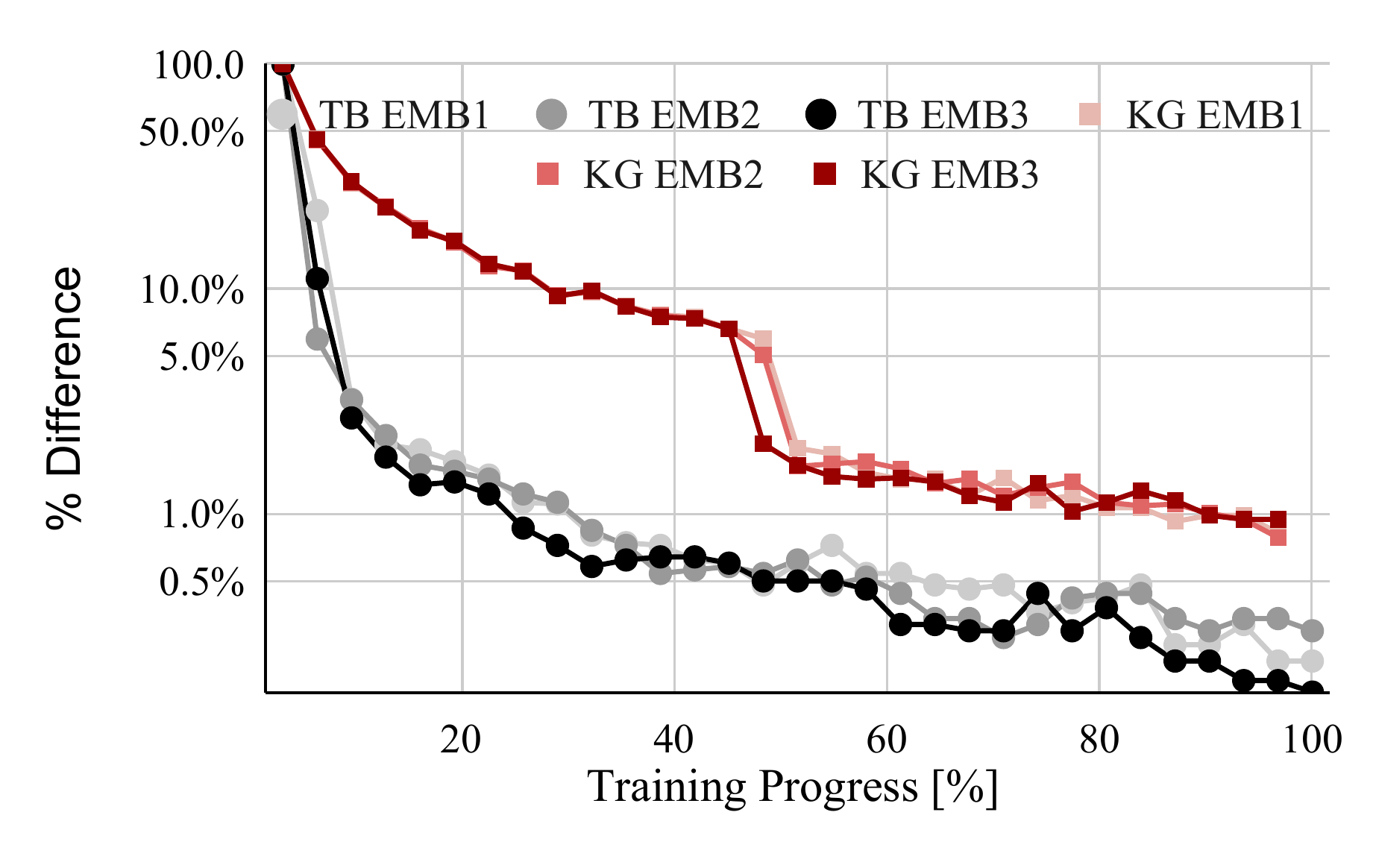}
    \caption{The set of frequently-accessed embedding rows over time stabilize at around 5\% and 50\% of the training run for Terabyte and Kaggle, respectively. 
    }
    \vspace{-8mm}
    \label{fig:freq-timeline}
\end{figure}

\begin{figure*}[t]
    \centering
    \includegraphics[width=0.47\textwidth]{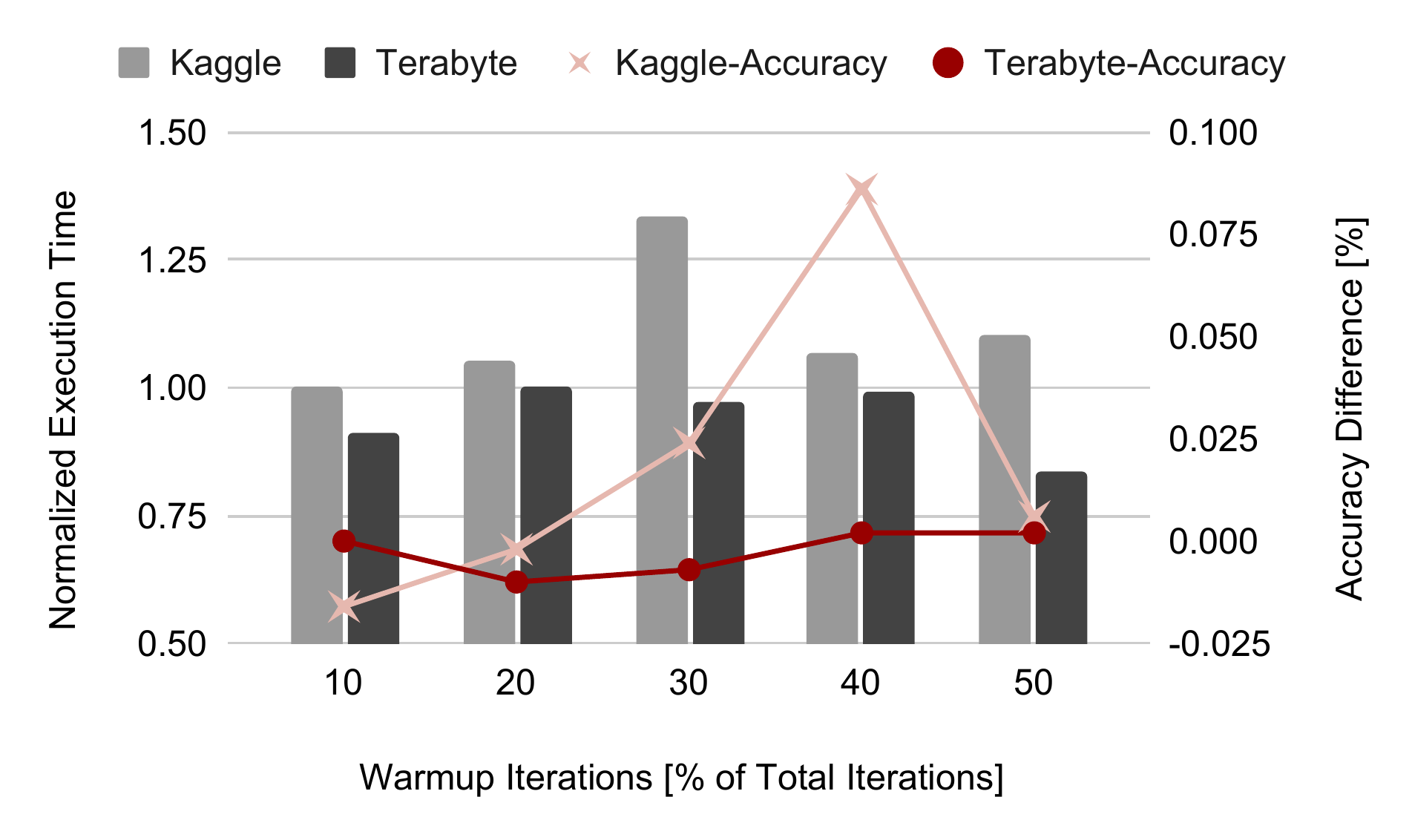}
    \includegraphics[width=0.47\textwidth]{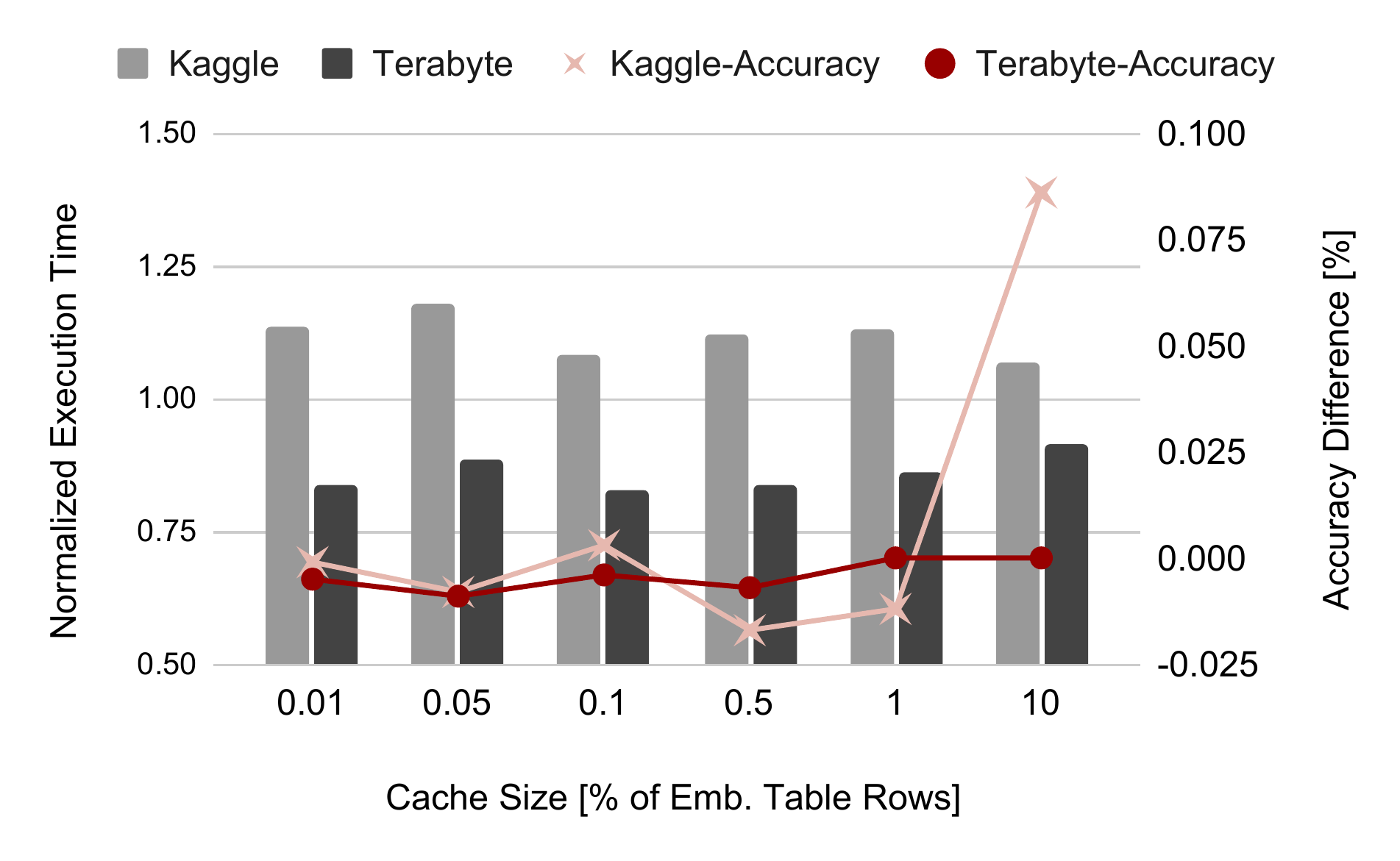}
    \vspace{-8mm}
    \caption{The impact of warm-up iterations on the TT-Rec training time and model accuracy (left). The model training time and accuracy over the cache sizes with respect to the embedding table (right). 
    }
    \label{fig:cache_warmup}
\end{figure*}

To quantify the efficiency of our TT-Rec implementation, we compare the performance of the TT-Embedding kernel with the baseline PyTorch EmbeddingBag operator~\cite{pytorch} and the state-of-the-art TT embedding library, called T3nsor~\cite{hrinchuk2020tensorized}, for word embedding of NLP use cases.
Figure~\ref{fig:tt-comp-cost} shows that our TT-Embedding implementation is order-of-magnitude more efficient than that of T3nsor, from the perspectives of compute (represented by circle) and memory (represented by square) requirement, over the number of rows in embedding table (x-axis). 
T3nsor decompresses embedding tables on the fly; thus, it requires the same amount of memory footprint during training as that of the PyTorch Embedding Bag operator. Our implementation achieves a memory footprint reduction of $\frac{\#Emb \space Rows}{Batch\space Size}$, 
yielding roughly 10,000$\times$ lower memory footprint requirement as compared to T3nsor and the PyTorch implementation. The overall training time of TT-Rec is on par with that of the baseline using the PyTorch Embedding Bag operator (Figure \ref{fig:kaggle-dlrm-time}; TT-Emb. of 3; Kaggle), and is 2.4$\times$ faster than T3nsor on average.




\subsection{Analysis for Feature Reuse Potential and TT-Rec Cache Performance}
\label{sec:cache_performance_analysis}
As described in \S\ref{sec:training-cache}, we introduce a cache structure in TT-Rec to capture the small subset of embedding rows with much higher degree of locality. Based on the available memory capacity on the training system, TT-Rec caches an uncompressed version of frequently-accessed rows to reduce the computation need. 

To illustrate the feature reuse potential in the context of TT-Rec, Figure~\ref{fig:freq-timeline} depicts the percentage of changes in the set of the most-frequently-accessed 10k embedding rows over the training run, for the three largest embedding tables (EMB1, EMB2, and EMB3). We count the cumulative row access frequencies every $3\%$ of training progress and measure the difference between each consecutive points (y-axis of Figure~\ref{fig:freq-timeline} in the log-scale). This difference is indicative of training phase stability: the lower the difference is, the more stable the set of frequently-accessed rows is. 




Figure~\ref{fig:cache_warmup}(a) shows the impact of warm-up iterations on the TT-Rec training time and model accuracy.
During the warm-up period, the cache structure is being filled up with most-frequently-accessed embedding vectors, using the aforementioned LFU replacement. 
Thus, the longer the warm-up period is, the better the TT-Rec cache captures the most-frequently-accessed embedding vectors and the higher the cache hit rate is for the remaining training iterations. 

For Kaggle, as the warm-up iterations increase from 10\% to 30\% of the total training iterations, the total  training time increases by 33\%. This is because the training time speedup from the slightly higher cache hit rate (from the warm-up period) is not sufficient to compensate the warm-up time overhead. As warm-up continues to increase beyond 30\% of the training, we start seeing the training time overhead to decrease. This happens as the hit rate improves with more warm-up iterations. 
In contrast, with Terabyte (the large industry-scale data set), TT-Rec cache can effectively and consistently reduce the model training time across the different warm-up iterations. It improves the end-to-end model training time by up-to 19\% with negligible accuracy impact. Overall, with the cache support, TT-Rec receives additional accuracy improvement of 0.09\% on Kaggle and 0.02\% on Terabyte.

Another important design parameter for TT-Rec is the size of the LFU cache. 
Figure~\ref{fig:cache_warmup}(b) shows the corresponding model training time and accuracy over the cache sizes ranging from 0.01\% to 10\% of the respective embedding tables. 
For Kaggle and Terabyte, devoting 0.01\% worth of the embedding table memory requirement is sufficient. 



\begin{figure}
    \vspace{-2.5mm}
    \centering
    \includegraphics[width=0.9\columnwidth]{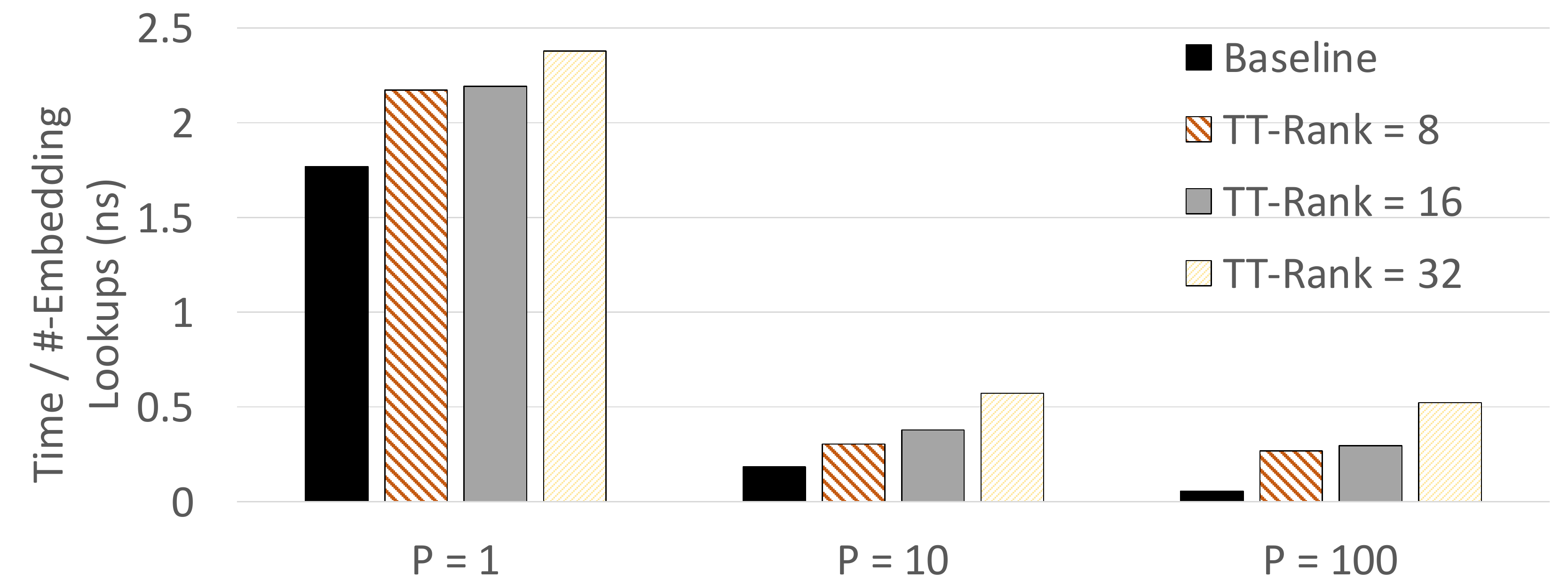}
    \vspace{-2.5mm}
    \caption{Performance of TT-Rec for MLPerf-DLRM (P of 1) and embedding-dominated DLRMs (P of 10 and 100).}
    \vspace{-0.25cm}
    \label{fig:ttrec-cache-b}
\end{figure}

\begin{figure}
    \centering
    \includegraphics[width=0.98\columnwidth]{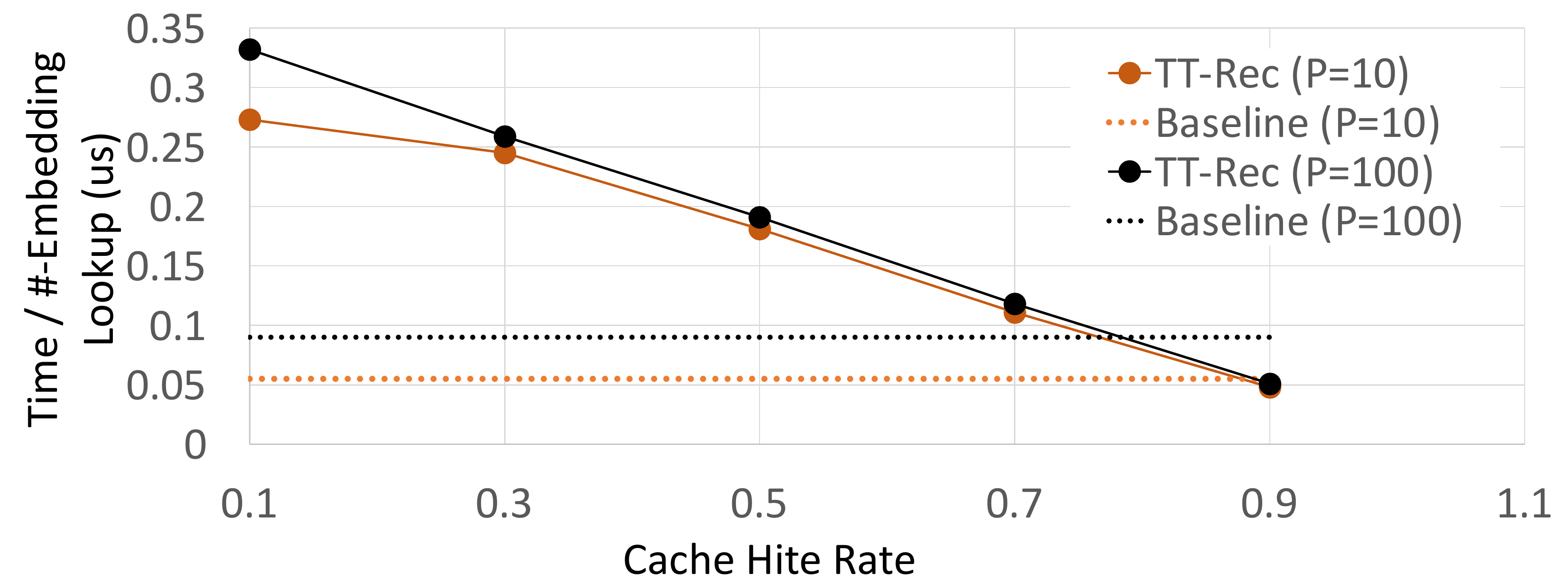}
    \vspace{-2.5mm}
    \caption{Performance comparison for TT-Rec's EmbeddingBag kernel and the baseline.}
    \vspace{-0.25cm}
    \label{fig:ttrec-cache-a}
\end{figure}

\subsection{TT-Rec Performance with Large Pooling Factors}
\label{sec:tt-rec-pooling}
To understand Rec-TT training performance for other categories of DLRMs that are embedding-dominated~\cite{Gupta2019,gupta:isca2020,ke:isca2020},
we develop a suite of microbenchmarks to synthetically generate embedding vectors with a polling factor (P) of 10 and 100. The pooling factor is defined as
the average number of embedding lookups required per training sample.
The larger P, the more embedding vectors are looked up and gathered, and the embedding operations
are more expensive. Both Kaggle and Terabyte correspond to P = 1. 

Figure~\ref{fig:ttrec-cache-b} compares the performance of
the non-cached TT-Rec kernel with PyTorch EmbeddingBag for P =[1, 10, 100]
across various TT-ranks. As seen in Figure~\ref{fig:ttrec-cache-b}, 
the performance per training sample is better as P increases.
This is because as the number of embedding lookups increase, the 
overhead of EmbeddingBag is amortized, in both the baseline and TT-Rec 
cases. Furthermore, the performance gap between the non-cached TT-Rec kernel 
and EmbeddingBag increases as P increases.
This is because there exists higher reuses of embedding vectors when P is larger
and EmbeddingBag benefits from such reuse.

To reduce the execution time of TT-Rec as P increases,
Figure~\ref{fig:ttrec-cache-a} compares the performance of TT-Rec
with caching enabled with EmbeddingBag. To quantify the  
timing performance per training sample, we synthetically generate data samples to control the cache hit rate.
As expected, as the cache hit rate increases (the x-axis), the performance of TT-Rec improves 
and eventually outperforms the baseline EmbeddingBag when the 
cache hit rate reaches 90\%.

\label{results}


\section{Related Work}
\label{related-work}
\textbf{General Model Compression Techniques:}
Among the many compression techniques for training, the commonly-used ones include magnitude pruning~\cite{zhu2017prune}, variational dropout~\cite{kingma2013auto}, and $l_0$ regularization~\cite{louizos2017learning}.
Other efforts propose to impose structured sparsity in model weights upfront~\cite{child2019generating,gray2017gpu}.
Such approaches can significantly reduce both training and inference cost, but have not been proven as effective solutions for deep learning recommendation models. 
Furthermore, while these methods are generally applicable, the prior works are fundamentally different from our proposed TT-decomposition approach.

\textbf{Embedding Table Compression Techniques:}
For embedding, the seminal work by Weinberger et al. examines feature hashing, that allows multiple elements to map to the same embedding vector; thus, it reduces the embedding space~\cite{weinberger2009feature}. However, hash collisions yield significant accuracy losses. 
For example, Zhao et al. observed an intolerable degree of accuracy loss if hashing were applied to TB-scale recommendation models~\cite{zhao2020distributed}.
Guan et al. propose a 4-bit  quantization scheme to compress  pre-trained models for inference~\cite{guan2019posttraining}. This design is feasible for recommendation inference although, quantization for training is more challenging and often comes with accuracy tradeoff.
Other works also explore the potential of low-rank approximation on the embedding tables~\cite{hrinchuk2020tensorized, ghaemmaghami2020training} but experience critical accuracy degradation. And none provide a computationally-efficient interface for industry-scale DLRMs.

\textbf{Tensorization:}
\label{sec:related:tensors}
Tensor methods have been extensively studied for compressing DNNs. One of the most common method is the Tucker factorization~\cite{cohen2016expressive}, which can generate high-quality DNNs when compressing fully-connected layers. 
Tensor Train (TT) and Tensor Ring (TR) decomposition techniques have been recently studied in the context of DNNs~\cite{Wang2018,Hawkins2019}.
But previous work has explored the accuracy trade-off for fully-connected and convolution layers only. In particular, TT decomposition offers a structural way to compress DNNs and thus is capable of preserving the DNN weights.
TR can preserve the weights with moderately lower compression ratios than that of TT~\cite{Wang2018}.
%
Despite interest in TT-based methods, to the best of our knowledge, ours is the first to consider them in the context of DLRMs.
In our analysis, we comprehensively study the design space of how memory size reduction, model quality, and training time overheads trade-off.
Finally, we also present an efficient implementation for TT-Rec, which we will release as open source upon the acceptance of this work. 



\section{Conclusion}
The spirit of TT-Rec is to use principled, parameterized algorithmic methods to help control the explosive demands on computational infrastructure.
This strategy complements innovations in the infrastructure itself.
TT-Rec specifically attacks the considerable memory requirements of embedding layers of modern recommendation models, whose memory requirements in industrial applications at scale can require hundreds of GBs to TBs of memory.
TT-Rec replaces otherwise large embedding tables with a sequence of matrix products, reducing the total model memory capacity requirement by 112 times with only a small amount of 13.9\% training time overhead while maintaining the same model accuracy as the baseline.
Such significant memory capacity reductions can be achieved with a relatively small increase in training time through clever caching strategies, making online recommendation model training more practical. 

\bibliography{tensor-train} 
\bibliographystyle{mlsys2020}

\clearpage
\begingroup
\let\clearpage\relax 
\onecolumn 

\begin{center}
    \textbf{APPENDIX}
\end{center}
\appendix

\section{Algorithm Details}
\label{sec:appendix_algorithm}
Algorithm~\ref{alg:tt-emb} and Algorithm~\ref{alg:tt-emb-back} below shows
the details of the forward and back-propagation algorithm
for TT-Rec embedding tables.
Algorithm~\ref{alg:sampled_gaussian} shows how the TT-cores are initialized using our proposed Sampled Gaussian method.

\begin{algorithm2e}[ht]
    \caption{Forward prop. of TT-Embedding}
    \label{alg:tt-emb}
    \SetAlgoLined
    \DontPrintSemicolon
    \While{$S\_idx < offsets[m]$}{   
        $E\_idx = min(offsets[S\_idx + B], offsets[m])$\;
        \For{$k = S\_idx$ to $E\_idx$}{
            $idx[j][k] = i_j^k$ in Eqn(~\ref{eqn:tt-embedding-bag})\;
            $a[k] = \&\mathcal{G}_1[idx[1][k]][0]$\;
            $b[k] = \&\mathcal{G}_0[idx[0][k]][0]$\;
            $c[k] = \&tr_0[k][0]$\;
            $a[k+B] = \&\mathcal{G}_2[idx[2][k]][0]$\;
            $b[k+B] = \&tr_0[k][0]$\;
            $c[k+B] = \&tr_1[k][0]$\;
        }
        \For{j =  0 to d-2}{
            \small{$\rhd$ Batched GEMM kernel calls}\;
            $c[jB:(j+1)B] = $
            $a[jB:(j+1)B] * b[jB:(j+1)B]$\;
        }
        \small{$\rhd$ Reduce embedding rows to output}\;
        $output[S\_idx:E\_idx] = \sum_{j = offsets[i]}^{offsets[i+1]} c[j]$\;
        $S\_idx = E\_idx$\;
    }
\end{algorithm2e}

\begin{algorithm2e}[ht]
    \caption{Backward prop. of TT-Rec Embeddings}
    \label{alg:tt-emb-back}
    \SetAlgoLined
    \DontPrintSemicolon
    \While{$S\_idx < offsets[m]$}{   
        $E\_idx = min(offsets[S\_idx + B], offsets[m])$\;
        Recompute for $tr_i$'s as in Algorithm~\ref{alg:tt-emb}\;
        \For{$k = S\_idx$ to $E\_idx$}{
            $idx[j][k] = i_j^k$ in Eqn(~\ref{eqn:tt-embedding-bag})\;
            $a0[k] = \&\mathcal{G}_0[idx[0][k]][0]$\;
            $b0[k] = \&tr_0[k][0]$\;
            $c0[k] = \&tr\_\mathcal{G}_1[k][0]$\;
            $a1[k] = \&tr_0[k][0]$\;
            $b1[k] = \&\mathcal{G}_1[idx[1][k]][0]$\;
            $c1[k] = \&tr\_\mathcal{G}_0[k][0]$\;
            $a0[k+B] = \& tr_0[k][0]$\;
            $b0[k+B] = \& dx$\;
            $c0[k+B] = \&tr\_\mathcal{G}_2[k][0]$\;
            $a1[k+B] = \& dx$\;
            $b1[k+B] = \&\mathcal{G}_2[idx[2][k]][0]$\;
            $c1[k+B] = \& tr_0[k][0]$\;
        }
        \For{j =  d-2 to 0}{
            \small{$\rhd$ Batched GEMM calls to compute $\partial\mathcal{G}_j$}\;
            $c0[jB:(j+1)B] = $
            $a0[jB:(j+1)B] * b0[jB:(j+1)B]$\;
            \small{$\rhd$ Batched GEMM calls to compute $\partial x$}\;
            $c1[jB:(j+1)B] = $
            $a1[jB:(j+1)B] * b1[jB:(j+1)B]$\;
            $\partial \mathcal{G}_j[idx[k]] += tr\_\mathcal{G}_j[k]$\;
        }
        $start\_idx = end\_idx$\;
    }
\end{algorithm2e}

\begin{algorithm2e}[hbt!]
    \DontPrintSemicolon
    \For{d = 0 to tt-dim}{
        $\mathcal{G}_d$ = random.normal(0,1) \;
        \For{each entry $\mathcal{G}_d(i, j, k, l)$ in $\mathcal{G}_d$}{
            \While{$\mathcal{G}_d(i, j, k, l) \le 2$}{
                $\mathcal{G}_d(i, j, k, l)$ = random.normal(0,1)\;
            }
        }
        $\mathcal{G}_d /= (\sqrt{1/3n})^{1/d}$\;
    }
    \caption{Sampled Gaussian initialization.}
    \label{alg:sampled_gaussian}
\end{algorithm2e}
\comment{
\newpage
\section{Related Work}

We describe the algorithm to initialize tensor train cores in a way that the product of them forms a semi-orthogonal matrix. We will first review the related work suggesting orthogonal initialization could be benificial to deep neural network training, and show how we can achieve that in a tensor train decomposed network.

\subsection{Provable Benefit of Orthogonal Initialization in Optimizing Deep Linear Networks}

This paper proved that for feedforward networks, if the weight matrices are independently from a uniform distribution over scaled orthogonal matrices, the network converges faster than Gaussian.
Let $X \in \mathbb{R}^{n \times d_x}$ be the input, and $W_1 \in \mathbb{R}{dx \times m}$, $W_i \in \mathbb{R}{m \times m}$ be the weight matrices satisfying
\begin{align*}
    & W_1^T(0) W_1(0) = m I_{d_x},\\
    & W_i^T(0) W_i(0) = W_i(0) W_i^T(0) = m I_m, i \ge 2
\end{align*}
The normalization factor $\alpha$ is set as $\alpha = \frac{1}{\sqrt{m^{L-1}d_y}}$ to ensure $\mathbb{E}[\|f(x; W_L(0), ..., W_1(0)\|^2] = \|x\|^2$ for any $x$.

Let $W^*\in arg min _{W} \|WX - Y\|_F$ and $l^* = \frac{1}{2}\|W^*X - Y\|_F^2$ be the optimal solution and the corresponding residual to the network. Denote $r = rank(X)$, $\kappa = \frac{\lambda_{max}(X^TX)}{\lambda_r(X^TX)}$, and $\Tilde{r} = \frac{\|X\|^2_F}{\|X\|^2}$, the main theorem of this paper is
\begin{theorem}
Suppose 
\[
m \ge C\Tilde{r}\kappa^2(d_y(1+\|W^*\|^2) + \log(r/\delta))\text{ and }m \ge d_x,
\]
for some $\delta\in (0,1)$ and a sufficiently large universal constant $C > 0$. Set the learning rate $\mathbb{}ta \le \frac{d_y}{2L\|X\|^2}$. Then with probability at least $1-\eta$ over the random initialization, we have
\begin{align}
    & l(0) - l^* \le O(1 + \frac{\log(r/\delta)}{d_y} + \|W^*\|^2) \|X\|_F^2, \\
    & l(t) - l^* \le (1- \frac{1}{2}\eta L \lambda_r(X^TX)/d_y)^t(l(0) - l^*), t = 0, 1, 2,\cdots
\end{align}
where l(t) is the objective value at iteration t.
\end{theorem}
Overall, it describes that when the weight matrices are initialized to be orthogonal and when $m$ is sufficiently large, with properly chosen learning rate, (i) the initialized point is close to the optimal solution, and (ii) the loss at each step is bounded and decreases exponentially. 

\subsection{Sample Matrix Independently from a Uniform Distribution over Orthogonal Matrices}
\begin{claim}
Let $A \in \mathbb{R}^{n\times m}$ be a matrix whose entries are i.i.d from a normal distribution, and let $QR = A$ be its QR decomposition. The resulting distribution of $Q$ is Haar measure on orthogonal matrices (invariant under symmetric operations).
\end{claim}
\begin{proof}
Suppose $x_1, x_2$ are the first 2 columns of matrix $X$, Gram Schmidt gives an orthogonal matrix with columns $\frac{x_1}{\|x_1\|}$ and $\frac{x_2-<x_1, x_2>\frac{x_1}{\|x_1\|}}{\|numerator\|}$. Assume the orthogonal matrix has distribution $X_{GS}$.

If $O$ is an orthogonal matrix, $OX$ has columns $Ox_1, Ox_2$. Applying Gram-Schmidt to $OX$,  we get 2 orthogonal vectors
\begin{align*}
    &\frac{Ox_1}{\|Ox_1\|} = \frac{Ox_1}{\|x_1\|}\\
    &\frac{Ox_2-<Ox_1, Ox_2>\frac{Ox_1}{\|Ox_1\|}}{\|numerator\|} = \frac{Ox_2-<x_1, x_2>\frac{Ox_1}{\|x_1\|}}{\|numerator\|}
\end{align*}
Therefore the distribution of $(OX)_{GS}$ is $OX_{GS} = X_{GS}$. This shows that the distribution $X_GS$ is invariant under orthogonal matrices.

\end{proof}

\section{Initialize Tensor Cores}
\subsection{Summary}
Let $W\in \mathbb{R}^{m \times n}$ be a matrix, and we apply 3-d TT decomposition on $W$ as follows
\[
\mathcal{W}((i_1, j_1), (i_2, j_2), (i_3, j_3)) = \mathcal{G}_1(:, i_1, j_1,:)\mathcal{G}_2(:, i_2, j_2, :) \mathcal{G}_d (:, i_3, j_3,:),
\]
where $\mathcal{G}_1\in \mathbb{R}^{1 \times I_1 \times J_1 \times R}$, $\mathcal{G}_2\in \mathbb{R}^{R \times I_2 \times J_2 \times R}$, and $\mathcal{G}_3\in \mathbb{R}^{R \times I_3 \times J_3 \times 1}$.
\begin{table}[h]
\centering
\begin{tabular}{|r|r|l|l|l|}
\hline
\multicolumn{2}{|l|}{Embedding Table size} & $\mathcal{G}_1$ shape   & $\mathcal{G}_2$ shape   & $\mathcal{G}_3$ shape   \\ \hline
10131227               & 16              & (1, 200, 2, $R$) & ($R$, 220, 2, $R$) & ($R$, 250, 4, 1) \\ \hline
8351593                & 16              & (1, 200, 2, $R$) & ($R$, 200, 2, $R$) & ($R$, 209, 4, 1) \\ \hline
7046547                & 16              & (1, 200, 2, $R$) & ($R$, 200, 2, $R$) & ($R$, 200, 4, 1) \\ \hline
5461306                & 16              & (1, 166, 2, $R$) & ($R$, 175, 2, $R$) & ($R$, 188, 4, 1) \\ \hline
2202608                & 16              & (1, 125, 2, $R$) & ($R$, 130, 2, $R$) & ($R$, 136, 4, 1) \\ \hline
286181                 & 16              & (1, 53, 2, $R$)  & ($R$, 72, 2, $R$)  & ($R$, 75, 4, 1)  \\ \hline
142572                 & 16              & (1, 50, 2, $R$)  & ($R$, 52, 2, $R$)  & ($R$, 55, 4, 1) \\ \hline
\end{tabular}
\caption{Size of 5 largest embedding tables and their TT decomposition in DLRM traing with Kaggle}
\end{table}

We want to initialize the tensor cores so that $W$ is uniformly random from the space of scaled semi-orthogonal matrix, i.e $W^T W = \alpha I$. The algorithm to initialize for each TT-core is as follows
\begin{enumerate}
    \item Generate $RJ_3$ orthogonal vectors of length $I_3$ $v_1, v_2, \dots, v_{RJ_3}$, set\\ $\mathcal{G}_3(r, :, j, 0)= v_{rJ_3+j}$
    \item Generate $RJ_2$ orthogonal vectors of length $RI_2$ $w_1, w_2, \dots, w_{RJ_2}$, set\\ $\mathcal{G}_2(r, : , j, :) = w_{rJ_2+ j}.reshape(I_2, R)$
    \item Generate $J_1$ orthogonal vectors of length $RI_1$ $u_1, u_2, \dots, u_{J_1}$, set\\ $\mathcal{G}_1(0, : , j, :) = w_{j}.reshape(I_1, R)$
\end{enumerate}
In step 2 and 3, the embedding table size ensures that $RJ_\alpha < RI_\alpha$ and therefore it is feasible to fill the tensor cores with $RJ_\alpha$ orthogonal vectors. However for $\mathcal{G}_3$, it is not guaranteed that $RJ_3 \ge I_3$ and so we have to choose tensor ranks carefully. Another possibility is to factorize the dimensions in a different way so $I_3$ can be larger. Yet it is unclear if this would affect TT accuracy.

\subsection{Proof}
Suppose $i = \sum_{\alpha = 1}^3 i_\alpha I_\alpha$, $j = \sum_{\alpha = 1}^3 j_\alpha J_\alpha$ be two indices in $W$, we compute each entry in $W^TW$
\begin{align}
    &W^T W(i,j) = \sum_{k = 1}^n W^T(i, k) W(k, j) = \sum_{k=1}^n W(k,i) W(k,j)\\
    & = \sum_{k = 1}^n \mathcal{G}_1(:, k_1, i_1,:)\mathcal{G}_2(:, k_2, i_2, :) \mathcal{G}_3 (:, k_3, i_3,:) (\mathcal{G}_1(:, k_1, j_1,:)\mathcal{G}_2(:, k_2, j_2, :) \mathcal{G}_3 (:, k_3, j_3,:))\\
    & = \sum_{k = 1}^n \mathcal{G}_1( k_1, i_1,:)\mathcal{G}_2(:, k_2, i_2, :) \mathcal{G}_3 (:, k_3, i_3) (\mathcal{G}_3 (:, k_3, j_3)^T\mathcal{G}_2(:, k_2, j_2, :)^T\mathcal{G}_1(k_1, j_1,:)^T  )\\
    & = \sum_{k = 1}^n \mathcal{G}_1( k_1, i_1,:)\mathcal{G}_2(:, k_2, i_2, :) (\mathcal{G}_3 (:, k_3, i_3) \mathcal{G}_3 (:, k_3, j_3)^T)\mathcal{G}_2(:, k_2, j_2, :)^T\mathcal{G}_1(k_1, j_1,:)^T 
\label{eqn:semi-ot}
\end{align}
As a first step, we need to initialize $\mathcal{G}_3$ such that
\begin{align}
    A = &\sum_{k_3}\mathcal{G}_3 (:, k_3, i_3) \mathcal{G}_3 (:, k_3, j_3)^T = 
    \begin{cases}
        I , & \mbox{ if } i_3 = j_3\\
        0, &\mbox{ if } i_3 \neq j_3
    \end{cases}
\end{align}

Note that in general, $\mathcal{G}_3$ is in shape $R \times n \times m$, where $R < n, n \in [50, 300]$ and $m \le 10$. If $i_3 = j_3$,
\begin{equation}
    A(p, q) = 
    \begin{cases}
        \sum_{k_3} \mathcal{G}_3(p, k_3, i_3)^2 = 1 & \mbox{, if } p = q,\\
        \sum_{k_3} \mathcal{G}_3(p, k_3, i_3) \mathcal{G}_3(q, k_3, i_3) = 0& \mbox{, if } p\neq q
    \end{cases}
\label{eqn:G3}
\end{equation}

To solve for ref{eqn:G3}, we need $\mathcal{G}_3(:, :, i)$ to be row-wise semi-orthogonal for each $i$. i.e. $\|\mathcal{G}_3(p, :, i)\| = 1$ $\forall p$, and $\mathcal{G}_3(p, :, i) \perp \mathcal{G}_3(q, :, i)$. \textbf{Therefore we need to initialize each $\mathcal{G}_3(:, :, i)$ from QR decomposition of a random matrix $\in \mathbb{R}^{n \times n}$ and adopt the first $R$ rows of $Q$.}

If $i_3 \neq j_3$,
\begin{equation}
    A(p, q) = 
    \begin{cases}
        \sum_{k_3} \mathcal{G}_3(p, k_3, i_3)\mathcal{G}_3(p, k_3, j_3) = 0 & \mbox{, if } p = q,\\
        \sum_{k_3} \mathcal{G}_3(p, k_3, i_3) \mathcal{G}_3(q, k_3, j_3) = 0& \mbox{, if } p\neq q
    \end{cases}
\end{equation}

Similarly, we need $\mathcal{G}_3(p, :, :)$ to be column-wise semi-orthogonal and can be obtained from Gram-Schmidt. For the second condition in (5), we need all the vectors of length $n$ to be orthogonal to each other(*). One could verify that under (*), (4) can also be satisfied. Therefore the algorithm is

\textbf{Use QR decomposition on a $n \times n$ matrix, reshape the $Q$ matrix to $R \times n \times m$ as $\mathcal{G}_3$ such that each vector of length $n$ is orthogonal to each other. }

\begin{figure*}[ht]
    \centering
    \includegraphics[width=\textwidth]{Figures/ortho-proof.png}
    \caption{Visualization of equation (8) and (9)}
\end{figure*}

Continue with equation \ref{eqn:semi-ot}), we have 
\begin{align*}
    &W^T W(i,j) =\\
    & \sum_{k = 1}^n \mathcal{G}_1( k_1, i_1,:)\mathcal{G}_2(:, k_2, i_2, :) (\mathcal{G}_3 (:, k_3, i_3) \mathcal{G}_3 (:, k_3, j_3)^T)\mathcal{G}_2(:, k_2, j_2, :)^T\mathcal{G}_1(k_1, j_1,:)^T \\
    & = \begin{cases}
    \sum_{k_1}\sum_{k_2}\mathcal{G}_1( k_1, i_1,:)\mathcal{G}_2(:, k_2, i_2, :) \mathcal{G}_2(:, k_2, j_2, :)^T\mathcal{G}_1(k_1, j_1,:)^T & \mbox{, if } i_3 = j_3\\
    0 &\mbox{, if } i_3 \neq j_3
    \end{cases}\\
    & = \begin{cases}
    \sum_{k_1}\mathcal{G}_1( k_1, i_1,:)(\sum_{k_2} \mathcal{G}_2(:, k_2, i_2, :) \mathcal{G}_2(:, k_2, j_2, :)^T)\mathcal{G}_1(k_1, j_1,:)^T &\mbox{, if } i_3 = j_3\\
    0 &\mbox{, if } i_3 \neq j_3
    \end{cases}
\end{align*}

We solve for the next condition
\begin{align*}
    &B = \sum_{k_2}\mathcal{G}_2(:, k_2, i_2, :) \mathcal{G}_2(:, k_2, j_2, :)^T = \begin{cases}
    \alpha I &\mbox{, if } i_2 = j_2\\
    0 &\mbox{, if } i_2 \neq j_2
    \end{cases}
\end{align*}
If $i_2 = j_2$
\begin{align*}
    B(p, q) &= 
    \begin{cases}
        \sum_{k_2} \sum_s \mathcal{G}_2(p, k_2, i_2, s) \mathcal{G}_2^T(s, k_2, i_2, p) & \mbox{, if } p = q,\\
        \sum_{k_2} \sum_s \mathcal{G}_2(p, k_2, i_2, s) \mathcal{G}_2^T(s, k_2, i_2, q) & \mbox{, if } p\neq q
    \end{cases}\\
    & = \begin{cases}
        \sum_{k_2} \sum_s \mathcal{G}_2(p, k_2, i_1, s)^2 = 1 & \mbox{, if } p = q,\\
        \sum_{k_2} \sum_s \mathcal{G}_2(p, k_2, i_2, s) \mathcal{G}_2(q, k_2, i_2, s) = 0 & \mbox{, if } p\neq q
    \end{cases}\\
\end{align*}
Similarly, this condition implies that $\mathcal{G}_2(:, k_2, i_2, :)$ is a row-wise orthogonal matrix for all $k_2, i_2$. 

If $i_2 \neq j_2$,
\begin{align*}
B(p, q) &= 
    \begin{cases}
        \sum_{k_2} \sum_s \mathcal{G}_2(p, k_2, i_1, s) \mathcal{G}_2^T(s, k_2, j_2, q) & \mbox{, if } p = q,\\
        \sum_{k_2} \sum_s \mathcal{G}_2(p, k_2, i_2, s) \mathcal{G}_2^T(s, k_2, j_2, q) & \mbox{, if } p\neq q
    \end{cases}\\
    & = \begin{cases}
        \sum_{k_2} \sum_s \mathcal{G}_2(p, k_2, i_1, s) \mathcal{G}_2(p, k_2, j_2, s) = 0 & \mbox{, if } p = q,\\
        \sum_{k_2} \sum_s \mathcal{G}_2(p, k_2, i_2, s) \mathcal{G}_2(q, k_2, j_2, s) = 0 & \mbox{, if } p\neq q
    \end{cases}
\end{align*}
We use QR decomposition to obtain $RJ_2$ (corresponding to dimension $p/q$ and $i/j$) orthogonal vectors of length $RI_2$ (corresponding to dimension $k_2$ and $s$), reshape and place each vector at $\mathcal{G}_2(p, :, i, :)$. We could verify that all the conditions above are satisfied.

\begin{align*}
    &W^T W(i,j) =\\
    & \sum_{k = 1}^n \mathcal{G}_1( k_1, i_1,:)\mathcal{G}_2(:, k_2, i_2, :) (\mathcal{G}_3 (:, k_3, i_3) \mathcal{G}_3 (:, k_3, j_3)^T)\mathcal{G}_2(:, k_2, j_2, :)^T\mathcal{G}_1(k_1, j_1,:)^T \\
    & = \begin{cases}
    \sum_{k_1}\mathcal{G}_1( k_1, i_1,:)\mathcal{G}_1(k_1, j_1,:)^T & \mbox{, if } i_3 = j_3, i_2 = j_2\\
    0 &\mbox{, if } i_3 \neq j_3 \text{ or } i_2 \neq j_2
    \end{cases}
\end{align*}

\begin{align}
    C = &\sum_{k_1}\mathcal{G}_1 (k_1, i_1, :) \mathcal{G}_1 (k_1, j_1, :)^T = 
    \begin{cases}
        1 , & \mbox{ if } i_1 = j_1\\
        0, &\mbox{ if } i_1 \neq j_1
    \end{cases}
\end{align}
Similar to the analysis for $\mathcal{G}_2$, we generate $J_1$ orthogonal vectors of length $I_1 R$, reshape and place them at $\mathcal{G}_1(:, j, :)$ for each $j$.

This completes our analysis.

\section{Empirical Results}
\begin{figure}[h]
    \centering
    \includegraphics[scale=.8]{Figures/ortho-ele-dist.png}
    \caption{Distribution of the full tensor entries generated by our algorithm}
\end{figure}
We tested the model accuracy with initializing all tensor cores as described in this document, and uncompressed embedding tables from uniform distribution. The weights of the MLPs are initialized from Gaussian. Overall the accuracy is not significantly improved compare to other initialization methods(Figure ref{fig:table}). In Figure ref{fig:accuracy} we plot the training accuracy curve of the uncompressed DLRM and TT-DLRMs. It seems like the final accuracy is largely dependent on the initial accuracy, and the learning curve of all models are roughly the same. 

Note: the sampled Gaussian method in the two figures refers to initializing values randomly for Gaussian(0,1) and regenerating values within range (-2,2). The distribution of the full tensor obtained by this method is plotted in Figure ref{fig:sample}.
\begin{figure*}[h]
    \centering
    \includegraphics[width=\columnwidth]{Figures/ortho-accuracy.png}
    \caption{Training accuracy over iterations}
    \label{fig:table}
\end{figure*}
\begin{figure}
    \centering
    \includegraphics[width=\columnwidth]{Figures/ortho-accu-table.png}
    \caption{Training loss and accuracy of DLRM initialized by different distributions}
    \label{fig:accuracy_table}
\end{figure}

}

\endgroup

\end{document}